 \documentclass{colt2016} 

\usepackage{bbold}
\usepackage{enumerate}

\title[Scale-free network optimization: foundations and algorithms]{Scale-free network  optimization: foundations and algorithms}
\usepackage{times}

 \coltauthor{\Name{Patrick Rebeschini} \Email{patrick.rebeschini@yale.edu}\\
 \Name{Sekhar Tatikonda} \Email{sekhar.tatikonda@yale.edu}\\
 \addr Yale Institute for Network Science, 17 Hillhouse Avenue, Yale University, New Haven, CT, 06511, USA}

\begin{document}

\maketitle

\thispagestyle{empty}


\begin{abstract}
We investigate the fundamental principles that drive the development of scalable algorithms for network optimization. Despite the significant amount of work on parallel and decentralized algorithms in the optimization community, the methods that have been proposed typically rely on strict separability assumptions for objective function and constraints. Beside sparsity, these methods typically do not exploit the \emph{strength} of the interaction between variables in the system. We propose a notion of correlation in constrained optimization that is based on the sensitivity of the optimal solution upon perturbations of the constraints. We develop a general theory of sensitivity of optimizers the extends beyond the infinitesimal setting. We present instances in network optimization where the correlation decays exponentially fast with respect to the natural distance in the network, and we design algorithms that can exploit this decay to yield dimension-free optimization. Our results are the first of their kind, and open new possibilities in the theory of local algorithms.
\end{abstract}



\section{Introduction}

Many problems in machine learning, networking, control, and statistics
can be posed in the framework of optimization.
Despite the significant amount of work on 
decomposition methods and 
decentralized algorithms in the optimization community, typically the methodologies being considered rely on strict separability assumptions  on the objective function and constraints, so that the problem can exactly decouple across components and each component can be handled by its own processing unit \citep{BT97,Boyd:2011}. These methods are insensitive to the \emph{strength} of interaction among variables in the system, and beside sparsity they typically do not exploit more refined structures.
On the other hand, probability theory has taught us that random variables need not to be independent for distributed methods to be engineered, and that notions of correlation decay can be exploited to develop scale-free algorithms \citep{G14}. 
This paper represents a first attempt to characterize the correlation among variables in network optimization, and to investigate how decay of correlations with respect to the natural distance of the network can be exploited to develop scalable computationally-efficient algorithms. The paper presents three main contributions.


\textbf{1) Sensitivity of optimal points: notion of correlation in optimization.} In Section \ref{sec:Local sensitivity for equality-constrained convex minimization} we develop a general theory on the sensitivity of optimal points in constrained convex optimization. We consider the problem of minimizing a convex function $x\rightarrow f(x)$ subject to $Ax = b$, for a certain matrix $A$ and vector $b\in\operatorname{Im}(A)$, where $\operatorname{Im}(A)$ denotes the image of $A$.
If the function $f$ is strongly convex, we show that the optimal point $b \rightarrow x^\star(b)$ is continuously differentiable along $\operatorname{Im}(A)$. We explicitly characterize the effect that perturbations have on the optimal solution as a function of the objective function $f$ and the constraint matrix $A$: given a differentiable function $\varepsilon\in\mathbb{R} \rightarrow b(\varepsilon)\in \operatorname{Im}(A)$, we establish an expression for $\frac{d x^\star(b(\varepsilon))}{d\varepsilon}$ in terms of the Moore-Penrose pseudoinverse of the matrix $A\Sigma(b(\varepsilon))A^T$, where $A^T$ is the transpose of $A$, and where $\Sigma(b)$ denotes the inverse of the Hessian of $f$ evaluated at $x^\star(b)$, namely, $\Sigma(b):=\nabla^2 f(x^\star (b))^{-1}$. We provide an interpretation of the derivatives of optimal points as a measure of the correlation between variables in the optimization procedure. Textbook results on the sensitivity analysis for optimization procedures are typically stated only with respect to the optimal objective function, i.e., $f(x^\star(b))$, which in general is a much more well-behaved object than the point where the optimum is attained, i.e., $x^\star(b)$. On the other hand, the literature on the sensitivity of optimal points (see \cite{CCMM2007} and reference therein) is only concerned with establishing infinitesimal perturbations locally, on a neighborhood of a certain $b\in\operatorname{Im}(A)$, while the theory that we develop extends to \emph{finite} perturbations as well via the fundamental theorem of calculus. The workhorse behind our results is Hadamard's global inverse function theorem. The details of the proofs involving Hadamard's theorem are presented in Appendix \ref{sec:Hadamard global inverse theorem}.

\textbf{2) Foundation of scale-free network optimization: decay of correlation.} As a paradigm for network optimization, in Section \ref{sec:Optimal Network Flow} we consider the widely-studied min-cost network flow problem, which has been fundamental in the development of the theory of polynomial-times algorithms for optimizations (see \cite{GSW12} and references therein, or \cite{Ahuja:MinCostFlow} for book reference). Here a directed graph $\vec G=(V,\vec E)$ is given, with its structure encoded in the vertex-to-edge incidence matrix $A\in\mathbb{R}^{V\times \vec E}$.
To each edge $e\in \vec E$ is associated a flow $x_e$ with a cost $f_e(x_e)$, and to each vertex $v\in V$ is associated an external flow $b_v$. The min-cost network flow problem consists in finding the flow $x^\star(b)\in\mathbb{R}^{\vec E}$ that minimizes the total cost $f(x):=\sum_{e\in \vec E} f_e(x)$, and that satisfies the conservation law $Ax=b$.
In this setting, the general sensitivity theory that we developed allows to characterize the optimal flow in terms of graph Laplacians; in fact, in this case the matrix $A\Sigma(b)A^T$ corresponds to the Laplacian of an undirected weighted graph naturally associated to $\vec G$.
To estimate the \emph{strength} of the correlation, we develop a general connection between the Moore-Penrose pseudoinverse of graph Laplacians and the Green's function of random walks on weighed graphs. To the best of our knowledge, this connection --- which we present as standalone in Appendix \ref{sec:Laplacians and random walks} --- has not been previously investigated in the literature.
This result allows us to get an upper bound for the correlation term that 
decays exponentially as a function of the graph distance between the edges that are considered and the set of vertices where the perturbation is localized. The rate of the decay is controlled by the second largest eigenvalue in magnitude of the corresponding random walk.
This phenomenon can be interpreted as a first manifestation of the decay of correlation principle in constrained optimization, resembling the decay of correlation property in statistical mechanics and probability theory first investigated in the seminal work of Dobrushin \citep{Dob70} (for book references see \cite{Sim93} and \cite{Geo11}).

\textbf{3) Scale-free algorithms.} Finally, in Section \ref{sec:Scale-free algorithm} we investigate applications of our theory to the field of local algorithms.
To illustrate the main principle behind scale-free algorithms, we consider the case when the solution $x^\star(b)$ is given and we want to compute the solution $x^\star(b+p)$ for the perturbed flow $b+p$, where $p$ is supported on a small subset $Z\subseteq V$. In this setting, we show that the decay of correlation \emph{structurally} exhibited by the min-cost network flow problem can be exploited to design algorithms that yield \emph{scale-free} optimization, in the sense that the computational complexity required to meet a certain precision level does not depend on the dimension of the network $\vec G$. We consider a localized version of the projected gradient descent algorithm, which only updates the edges in a subgraph of $\vec G$ whose vertex set contains $Z$. 
The correlation decay property encodes the fact that when the external flow is locally perturbed 
it suffices to recompute the solution only for the part of the network that is ``mostly affected" by this perturbation, i.e., the set of nodes that have a distance at most $r$ from the perturbation set $Z$, where the radius $r$ is tuned to meet the desired level of error tolerance, given the size of the perturbation. Hence the savings in the computational complexity compared to global algorithms that update the solution at every edge in $\vec G$. The theory that we develop in the context of the min-cost network flow problem hints to a general framework to study the trade-off between statistical accuracy and computational complexity for local algorithms in optimization. Our results are the first of their kind, and represent a building block to develop more sophisticated algorithms to exploit decay of correlation in more general instances of network optimization. The proof of the results in Section \ref{sec:Scale-free algorithm} are given in Appendix \ref{app:error localized}.

\begin{remark}[Connection with previous work]
Some of the results presented in this paper will appear in a weaker form and without full proofs in \cite{RT16}. There, the sensitivity analysis is developed for matrices $A$'s that are full row rank, so that the matrix $A\Sigma(b)A^T$ is invertible under the assumption that $f$ is strongly convex. In the current work we relax this assumption and we provide results in terms of the pseudoinverse of $A\Sigma(b)A^T$. Moreover, the current paper presents the full details of the proof which involve Hadamard's global inverse function theorem (Appendix \ref{sec:Hadamard global inverse theorem}). Also the min-cost network flow problem was previously investigated in \cite{RT16}, albeit in a more restrictive fashion through the connection with \emph{killed} random walks. The current paper develops a more general theory of correlation for optimization in terms of graph Laplacians and Green's functions of ordinary (i.e., not killed) random walks on graphs (Appendix \ref{sec:Laplacians and random walks}).
The difference is crucial as far as the results on the decay of correlation property are concerned, as the second largest eigenvalue in magnitude of random walks on graphs is typically much more well-behaved than the largest eigenvalue of killed random walks, as far as the dependence with the dimension is concerned.
The algorithmic part of this paper (Section \ref{sec:Scale-free algorithm} and Appendix \ref{app:error localized}) is completely new.
\end{remark}

\begin{remark}[Notation]
\label{rem:notation}
Throughout, for a given real matrix $M$, we denote by $M^T$ its transpose, by $M^{-1}$ its inverse, and by $M^{+}$ its Moore-Penrose pseudoinverse. We denote by $\operatorname{Ker}(M):=\{x:Mx=0\}$ and $\operatorname{Im}(M):=\{y: y=Mx \text{ for some $x$}\}$ the kernel and the image of $M$, respectively. Given an index set $\mathcal{I}$ and subsets $K,L\subseteq\mathcal{I}$, if $M\in\mathbb{R}^{\mathcal{I}\times \mathcal{I}}$, we let $M_{K,L}\in\mathbb{R}^{K\times L}$ denote the submatrix corresponding to the rows of $M$ indexed by $K$ and the columns of $M$ indexed by $L$. We use the notation $I$ to indicate the identity matrix, $\mathbb{1}$ to indicate the all-one vector (or matrix), and $\mathbb{0}$ to indicate the all-zero vector (or matrix), whose sizes will be implied by the context. Given a vector $x\in\mathbb{R}^\mathcal{I}$, we denote by $x_i$ its $i$-th component, and we let $\| x\| := (\sum_{i\in \mathcal{I}} x^2_i)^{1/2}$ denote its $\ell_2$-norm. Given a subset $K\subseteq\mathcal{I}$ we define the localized $\ell_2$-norm on $K$ as $\| x\|_K := (\sum_{i\in K} x^2_i)^{1/2}$. Clearly, $\| x\|_\mathcal{I} = \| x\|$. We use the notation $|K|$ to denote the cardinality of $K$.
If $\vec G = (V, \vec E)$ denotes a directed graph with vertex set $V$ and edge set $\vec E$, we let $G = (V, E)$ represent the undirected graph naturally associated to $\vec G$, namely, $\{u,v\}\in E$ if and only if either $(u,v)\in \vec E$ or $(v,u)\in \vec E$.
\end{remark}


\section{Sensitivity of optimal points: notion of correlation in optimization}
\label{sec:Local sensitivity for equality-constrained convex minimization}

Let $\mathcal{V}$ be a finite set --- to be referred to as the ``variable set" --- and let $f:\mathbb{R}^\mathcal{V}\rightarrow \mathbb{R}$ be a strictly convex function, twice continuously differentiable. Let $\mathcal{F}$ be a finite set --- to be referred to as the ``factor set" --- and let $A\in\mathbb{R}^{\mathcal{F}\times \mathcal{V}}$.
Consider the following optimization problem over $x\in\mathbb{R}^\mathcal{V}$:
\begin{align*}
	\begin{aligned}
		\text{minimize }\quad   & f(x)\\
		\text{subject to }\quad & Ax = b,
	\end{aligned}
\end{align*}
for $b\in\operatorname{Im}(A)$, so that the feasible region is not empty.
Throughout this paper we think of the function $f$ and the matrix $A$ as fixed, and we consider the solution of the optimization problem above as a function of the vector $b\in\operatorname{Im}(A)$.
By strict convexity, this problem clearly has a unique optimal solution, which we denote by
$$
	x^\star(b) := {\arg\min}\left\{ f(x) : x\in\mathbb{R}^\mathcal{V},A x = b \right\}.
$$


Theorem \ref{thm:comparisontheorem} below provides a characterization of the way a perturbation of the constraint vector $b$ along the subspace $\operatorname{Im}(A)$ affects the optimal solution $x^\star(b)$, in the case when the function $f$ is strongly convex.
In textbooks, results on the sensitivity analysis for optimization procedures are typically stated only with respect to the optimal objective function, i.e., $f(x^\star(b))$, not with respect to the point where the optimum is attained, i.e., $x^\star(b)$. See \cite{Boyd}, for instance. The reason is that the optimal value typically behaves much more nicely with respect to perturbations than the optimizer itself. In case of linear programming when $f$ is linear, for instance, it is known that the optimal solution is differentiable upon perturbations, while the optimal point might jump as it is restricted to be on the extreme points of the feasible polyhedron \citep{bertsimas-LPbook}. On the other hand, the literature on the sensitivity of optimal points is only concerned with infinitesimal perturbations (see \cite{CCMM2007} and reference therein). The theory that we develop, instead, extends to finite perturbations as well, as we show that if $f$ is strongly convex then the optimal point $x^\star$ is \emph{continuously} differentiable along the \emph{entire} subspace $\operatorname{Im}(A)$, which allows the use of the fundamental theorem of calculus to get finite-difference statements (the results in Section \ref{sec:Scale-free algorithm} rely heavily on this fact). The workhorse that allows us to establish global results is Hadamard's global inverse function theorem. We now present the main result on the sensitivity of optimal points, together with the main outline of its proof. The technical details involving Hadamard's theorem are given in Appendix \ref{sec:Hadamard global inverse theorem}.


\begin{theorem}[Sensitivity of the optimal point]\label{thm:comparisontheorem}
Let $f:\mathbb{R}^\mathcal{V}\rightarrow \mathbb{R}$ be a strongly convex function, twice continuously differentiable.
Let $A\in\mathbb{R}^{\mathcal{F}\times \mathcal{V}}$. Define the function
$$
	x^\star:b\in \operatorname{Im}(A)\subseteq\mathbb{R}^\mathcal{F}\longrightarrow
	x^\star(b) := {\arg\min}\left\{ f(x) : x\in\mathbb{R}^\mathcal{V},A x = b \right\}
	\in\mathbb{R}^{\mathcal{V}}.
$$
For each $b\in \operatorname{Im}(A)$, let $\Sigma(b):=\nabla^2 f(x^\star (b))^{-1}$ and define
$$
	D(b)
	:=
	\Sigma(b)A^T
	\left(A\Sigma(b)A^T\right)^{+}.
$$
Then, $x^\star$ is continuously differentiable along the subspace $\operatorname{Im}(A)$, and given a differentiable function $\varepsilon\in\mathbb{R} \rightarrow b(\varepsilon)\in \operatorname{Im}(A)$, we have
$$
	\frac{d x^\star(b(\varepsilon))}{d \varepsilon} 
	= D(b(\varepsilon)) \frac{d b(\varepsilon)}{d \varepsilon}.
$$
\end{theorem}



\begin{proof}
The Lagrangian of the optimization problem is the function $\mathcal{L}$ from $\mathbb{R}^\mathcal{V}\times\mathbb{R}^\mathcal{F}$ to $\mathbb{R}$ defined as
$$
	\mathcal{L}(x,\nu) := f(x) + \sum_{a\in \mathcal{F}} \nu_a (A^T_a x - b_i),
$$
where $A^T_i$ is the $i$-th row of the matrix $A$ and $\nu = (\nu_a)_{a\in\mathcal{F}}$ is the vector formed by the Lagrangian multipliers.
Let us define the function $\Phi$ from $\mathbb{R}^\mathcal{V}\times\mathbb{R}^\mathcal{F}$ to $\mathbb{R}^\mathcal{V}\times\mathbb{R}^\mathcal{F}$ as
$$
	\Phi
	( 
	x,
	\nu 
	)
	:=
	\left( 
	\begin{array}{c}
	\nabla_x \mathcal{L}(x,\nu) \\
	A x 
	\end{array} \right)
	=
	\left( 
	\begin{array}{c}
	\nabla f(x) + A^T\nu \\
	A x 
	\end{array} \right).
$$
For any fixed $\varepsilon\in\mathbb{R}$, as the constraints are linear, the Lagrange multiplier theorem says that for the unique minimizer $x^\star(b(\varepsilon))$ there exists $\nu'(b(\varepsilon))\in\mathbb{R}^\mathcal{F}$ so that
\begin{align}
	\Phi
	( 
	x^\star(b(\varepsilon)),
	\nu'(b(\varepsilon))
	)
	=
	\left( 
	\begin{array}{c}
	\mathbb{0} \\
	b(\varepsilon)
	\end{array} \right).
	\label{diff}
\end{align}
As $A^T(\nu+\mu) = A^T\nu$ for each $\mu\in\operatorname{Ker}(A^T)$, the set of Lagrangian multipliers $\nu'(b(\varepsilon))\in\mathbb{R}^\mathcal{F}$ that satisfies \eqref{diff} is a translation of the null space of $A^T$. We denote the unique translation vector by $\nu^\star(b(\varepsilon))\in\operatorname{Im}(A)$.
By Hadamard's global inverse function theorem, as shown in Lemma \ref{lem:Hadamard global inverse theorem} in Appendix \ref{sec:Hadamard global inverse theorem}, the restriction of the function $\Phi$ to $\mathbb{R}^\mathcal{V}\times \operatorname{Im}(A)$ is a $C^1$ diffeomorphism, namely, it is continuously differentiable, bijective, and its inverse is also continuously differentiable.
In particular, this means that the functions $x^\star:b\in\operatorname{Im}(A) \rightarrow x^\star(b)\in\mathbb{R}^\mathcal{V}$ and $\nu^\star:b\in\operatorname{Im}(A)\rightarrow \nu^\star(b)\in\operatorname{Im}(A)$ are continuously differentiable along the subspace $\operatorname{Im}(A)$. Differentiating both sides of \eqref{diff} with respect to $\varepsilon$, we get, by the chain rule,
$$
	\left( \begin{array}{cc}
	H & A^T \\
	A & \mathbb{0} 
	\end{array} \right)
	\left( \begin{array}{c}
	x'\\
	\tilde \nu
	\end{array} \right)
	=
	\left( \begin{array}{c}
	\mathbb{0}\\
	\frac{d b(\varepsilon)}{d \varepsilon}
	\end{array} \right),
$$
where $H:=\nabla^2 f(x^\star (b(\varepsilon)))$, $x' := \frac{d x^\star(b(\varepsilon))}{d \varepsilon}$, $\tilde \nu := \frac{d \nu^\star(b(\varepsilon))}{d \varepsilon}$.
As the function $f$ is strongly convex, the Hessian $\nabla^2 f(x)$ is positive definite for every $x\in\mathbb{R}^\mathcal{V}$, hence it is invertible for every $x\in\mathbb{R}^\mathcal{V}$. Solving the linear system for $x'$ first, from the first equation $Hx'+ A^T\tilde\nu =\mathbb{0}$ we get $x' = - H^{-1} A^T \tilde\nu$. Substituting this expression in the second equation $Ax' = \frac{d b(\varepsilon)}{d \varepsilon}$, we get $L\tilde \nu = - \frac{d b(\varepsilon)}{d \varepsilon}$, where $L:=AH^{-1}A^T$.
The set of solutions to $L\tilde \nu = - \frac{d b(\varepsilon)}{d \varepsilon}$ can be expressed in terms of the pseudoinverse of $L$ as follows (see \cite{Barata:2012fk}[Theorem 6.1], for instance):
$$
	\left\{\tilde \nu\in\mathbb{R}^\mathcal{F} : L\tilde \nu = - \frac{d b(\varepsilon)}{d \varepsilon}\right\}
	=
	-L^+\frac{d b(\varepsilon)}{d \varepsilon} + \operatorname{Ker}(L).
$$
We show that $\operatorname{Ker}(L) = \operatorname{Ker}(A^T)$. We show that $L\nu =\mathbb{0}$ implies $A^T\nu =\mathbb{0}$, as the opposite direction trivially holds. In fact, let $A' := A\sqrt{H^{-1}}$, where $\sqrt{H^{-1}}$ if the positive definite matrix that satisfies $\sqrt{H^{-1}} \sqrt{H^{-1}} = H^{-1}$. The condition $L\nu = A' A'^T \nu = \mathbb{0}$ is equivalent to $A'^T \nu\in\operatorname{Ker}(A')$. At the same time, clearly, $A'^T \nu\in\operatorname{Im}(A'^T)$. However, $\operatorname{Ker}(A')$ is orthogonal to $\operatorname{Im}(A'^T)$, so it must be $A'^T\nu =\mathbb{0}$ which implies $A^T\nu =\mathbb{0}$ as $\sqrt{H^{-1}}$ is positive definite. By \cite{Barata:2012fk}[Prop. 3.3] it follows that $\operatorname{Im}(L^+) = \operatorname{Ker}(L)^\perp = \operatorname{Ker}(A^T)^\perp = \operatorname{Im}(A)$, so $\tilde\nu = -L^+\frac{d b(\varepsilon)}{d \varepsilon}$ is the unique solution to $L\tilde \nu = - \frac{d b(\varepsilon)}{d \varepsilon}$ that belongs to $\operatorname{Im}(A)$. 
Substituting this expression into $x' = - H^{-1} A^T \tilde\nu$,
we finally get
$
	x' = H^{-1} A^T L^+ \frac{d b(\varepsilon)}{d \varepsilon}.
$
The proof follows as $\Sigma(b)=H^{-1}$.
\end{proof}

Theorem \ref{thm:comparisontheorem} characterizes the behavior of the optimal point $x^\star(b)$ upon perturbations of the constraint vector $b$ along the subspace $\operatorname{Im}(A)\subseteq \mathbb{R}^\mathcal{F}$. If the matrix $A$ is full row rank, i.e., $\operatorname{Im}(A) = \mathbb{R}^\mathcal{F}$, then the optimal point $x^\star$ is everywhere continuously differentiable, and we can compute its gradient. In this case the statement of Theorem \ref{thm:comparisontheorem} simplifies, as $(A\Sigma(b)A^T)^+=(A\Sigma(b)A^T)^{-1}$. The following corollary makes this precise.

\begin{corollary}[Sensitivity of the optimal point, full rank case]
\label{cor:fullrank}
Consider the setting of Theorem \ref{thm:comparisontheorem}, with the matrix $A\in\mathbb{R}^{\mathcal{F}\times \mathcal{V}}$ having full row rank, i.e., $\operatorname{Im}(A) = \mathbb{R}^\mathcal{F}$. 
Then, the function $b\in \mathbb{R}^\mathcal{F}\rightarrow x^\star(b)\in\mathbb{R}^{\mathcal{V}}$ is continuously differentiable and 
$$
	\frac{d x^\star(b)}{d b}
	= D(b)
	=
	\Sigma(b)A^T \left(A\Sigma(b)A^T\right)^{-1}.
$$
\end{corollary}

\begin{proof}
The proof follows immediately from Theorem \ref{thm:comparisontheorem}, once we notice that the matrix $L(b):=A\Sigma(b)A^T$ is positive definite for every $b \in\mathbb{R}^{\mathcal{F}}$, hence invertible, and $L(b)^+=L(b)^{-1}$. To see this, let $\nu\in \mathbb{R}^\mathcal{F}, \nu\neq \mathbb{0}$. Since $A^T$ has full column rank, we have $\rho:=A^T \nu \neq \mathbb{0}$, and as $\nabla^2f(x^\star (b))$ is positive definite by the assumption of strong convexity, also its inverse $\Sigma(b)$ is positive definite and we have
$
	\nu^TL(b)\nu
	=\nu^T A \Sigma(b) A^T\nu = \rho^T \Sigma(b)\rho > 0.
$
\end{proof}

If the matrix $A$ is full row rank, then the quantity $\frac{\partial x^\star(b)_i}{\partial b_a}$ represents a natural notion of the correlation between variable $i\in\mathcal{V}$ and factor $a\in\mathcal{F}$ in the optimization procedure, and the quantity $D(b)_{ia}$ in Corollary \ref{cor:fullrank} characterizes this correlation as a function of the constraint matrix $A$, the objective function $f$, and the optimal solution $x^\star (b)$. Theorem \ref{thm:comparisontheorem} allows us to extend the notion of correlation between variables and factors to the more general case when the matrix $A$ is not full rank. As an example, let $b,p\in \operatorname{Im}(A)$, and assume that $p$ is supported on a subset $F\subseteq \mathcal{F}$, namely, $p_a\neq 0$ if and only if $a\in F$. Define $b(\varepsilon) := b+\varepsilon p$. Then, the quantity $\frac{d x^\star(b(\varepsilon))_i}{d\varepsilon}$ measures how much a perturbation of the constraints in $F$ affects the optimal solution at $i\in \mathcal{V}$, hence it can be interpreted as a measure of the correlation between variable $i$ and the factors in $F$, which is characterized by the quantity $(D(b(\varepsilon)) \frac{d b(\varepsilon)}{d \varepsilon})_{i}=\sum_{a\in F} D(b(\varepsilon))_{ia} p_a$ in Theorem \ref{thm:comparisontheorem}.

\begin{remark}[Previous literature on notions of correlation in optimization]
There is only one paper that we are aware of where notions of correlation among variables in optimization procedures have been considered, which is \cite{MVR10}. In this paper the authors use a notion of correlation similar to the one that we are proposing to prove the convergence of the min-sum message passing algorithm to solve the class of separable \emph{unconstrained} convex optimization problems. Yet, in that work correlations are simply regarded as a tool to prove convergence guarantees for the specific algorithm at hand, and no general theory is built around them.
On the other hand, the need to address diverse large-scale applications in the optimization and machine learning domains prompts to investigate the \emph{foundations} of notions of correlation in optimization, and to develop a general theory that can inspire a principled use of these concepts for local algorithms. This is one of the main goal of our paper.
\end{remark}


In the next section we investigate the notion of correlation just introduced in the context of network optimization, when the constraints naturally reflect a graph structure, and we investigate the behavior of the correlations as a function of the natural distance in the graph.

\section{Foundation of scale-free network optimization: decay of correlation}\label{sec:Optimal Network Flow}

As a paradigm for network optimization, we consider the minimum-cost network flow problem, a cornerstone in the development of the theory of polynomial-times algorithms for optimizations. We refer to \cite{GSW12} for an account of the importance that this problem has had in the field of optimization, and to \cite{Ahuja:MinCostFlow} for book reference.

Consider a directed graph $\vec{G}:=(V,\vec{E})$, 
with vertex set $V$ and edge set $\vec{E}$, with no self-edges and no multiple edges. Let $G=(V,E)$ be the undirected graph naturally associated with $\vec{G}$, that is, $\{u,v\}\in E$ if and only if either $(u,v)\in \vec{E}$ or $(v,u)\in \vec{E}$. Without loss of generality, assume that $G$ is connected, otherwise we can treat each connected component on its own. For each $e\in \vec{E}$ let $x_e$ denote the flow on edge $e$, with $x_e>0$ if the flow is in the direction of the edge, $x_e<0$ if the flow is in the direction opposite the edge. For each $v\in V$ let $b_v$ be a given external flow on the vertex $v$: $b_v>0$ represents a source where the flow enters the vertex, whereas $b_v<0$ represents a sink where the flow enters the vertex. Assume that the total of the source flows equals the total of the sink flows, that is, $\mathbb{1}^Tb = \sum_{v\in V} b_v = 0$, where $b=(b_v)_{v\in V}\in\mathbb{R}^V$ is the flow vector. We assume that the flow satisfies a conservation equation so that at each vertex the total flow is zero. This conservation law can be expressed as $A x = b$, where $A\in\mathbb{R}^{V\times \vec{E}}$ is the \emph{vertex-to-edge incidence matrix} defined as
$$
	A_{ve}
	:=
	\begin{cases}
	1 & \text{if edge } e \text{ leaves node } v,\\
	-1 & \text{if edge } e \text{ enters node } v,\\
	0 & \text{otherwise}.
	\end{cases}
$$

\noindent For each edge $e\in \vec{E}$ let $f_e:\mathbb{R}\rightarrow\mathbb{R}$ be its associated cost function, assumed to be strongly convex and twice continuously differentiable.
The min-cost network flow problem reads
\begin{align*}
\begin{aligned}
	\text{minimize }\quad   & f(x) := \sum_{e\in \vec{E}} f_e(x_e)\\
	\text{subject to }\quad & Ax = b.
\end{aligned}
\end{align*}
It can be shown that since $G$ is connected $\operatorname{Im}(A)$ consists of all vectors orthogonal to the vector $\mathbb{1}$, i.e., $\operatorname{Im}(A) = \{ y \in \mathbb{R}^V: \mathbb{1}^T y = 0 \}$. See \cite{citeulike:12634920}, for instance.
Henceforth, for each $b\in\mathbb{R}^V$ such that $\mathbb{1}^Tb=0$, we let $x^\star(b)$ denote the unique optimal point of the network flow problem.

We first apply the sensitivity theory developed in Section \ref{sec:Local sensitivity for equality-constrained convex minimization} to characterize the correlation between vertices (i.e., factors) and edges (i.e., variables) in the network flow problem. Then, we investigate the behavior of these correlations in terms of the natural distance on the graph $G$.

\subsection{Correlation in terms of graph Laplacians}
In the setting of the min-cost network flow problem, Theorem \ref{thm:comparisontheorem} immediately allows us to characterize the derivatives of the optimal point $x^\star$ along the subspace $\operatorname{Im}(A)$ as a function of graph Laplacians, as we now discuss. 
For $b\in\mathbb{R}^V$ such that $\mathbb{1}^Tb=0$, let $\Sigma(b) := \nabla^2 f(x^\star (b))^{-1}\in\mathbb{R}^{\vec{E}\times \vec{E}}$, which is a diagonal matrix with entries given by, for each $e\in \vec{E}$,
$$
	\sigma(b)_{e} := \Sigma(b)_{ee}:=\left(\frac{\partial^2 f_e(x^\star(b)_e)}{\partial x_e^2}\right)^{-1} > 0.
$$
Each term $\sigma(b)_{e}$ is strictly positive as $f_e$ is strongly convex by assumption.
Let $W(b)\in\mathbb{R}^{V\times V}$ be the symmetric matrix defined as follows, for each $u,v\in V$,
\begin{align*}
	W(b)_{uv}
	:=
	\begin{cases}
	\sigma(b)_{e} & \text{if } e=(u,v) \in \vec{E} \text{ or } e=(v,u) \in \vec{E},\\
	0 & \text{otherwise},
	\end{cases}
\end{align*}
and let $D(b)\in\mathbb{R}^{V\times V}$ be the diagonal matrix with entries given by, for each $v\in V$,
\begin{align*}
	d(b)_{v} := D(b)_{vv} := \sum_{u\in V} W(b)_{vu}.
\end{align*}
Let $L(b):= D(b)-W(b)$ be the graph Laplacian of the undirected weighted graph $(V,E,W(b))$, where to each edge $e=\{u,v\}\in E$ is associated the weight $W(b)_{uv}$.
A direct application of Theorem \ref{thm:comparisontheorem} shows that the derivatives of the optimal point $x^\star$ along the subspace $\operatorname{Im}(A)$ can be expressed in terms of the Moore-Penrose pseudoinverse of $L(b)$.

\begin{lemma}[Sensitivity for min-cost network flow problem]\label{lem:comparisontheoremnetworkflow}
For $b\in\mathbb{R}^V$ such that $\mathbb{1}^Tb=0$, let
$$
	D(b)
	:=
	\Sigma(b)A^T L(b)^{+}.
$$
Then, $x^\star$ is continuously differentiable along the subspace $\operatorname{Im}(A)$, and given a differentiable function $\varepsilon\in\mathbb{R} \rightarrow b(\varepsilon)\in \operatorname{Im}(A)$, we have
$$
	\frac{d x^\star(b(\varepsilon))}{d \varepsilon} 
	= D(b(\varepsilon)) \frac{d b(\varepsilon)}{d \varepsilon}.
$$
\end{lemma}

\begin{proof}
The proof follows immediately from Theorem \ref{thm:comparisontheorem}, upon choosing variable set $\mathcal{V} := \vec{E}$ and factor set $\mathcal{F} := V$, and noticing that $A\Sigma(b)A^T = L(b)$.
\end{proof}

Let $b,p\in\mathbb{R}^V$ such that $\mathbb{1}^Tb=\mathbb{1}^Tp =0$, and assume that $p$ is supported on a subset $Z\subseteq V$, namely, $p_v\neq 0$ if and only if $v\in Z$. Define $b(\varepsilon) := b+\varepsilon p$. Then, as discussed in Section \ref{sec:Local sensitivity for equality-constrained convex minimization}, the quantity $\frac{d x^\star(b(\varepsilon))_e}{d\varepsilon}$ can be interpreted as a measure of the correlation between edge $e\in \vec{E}$ and the vertices in $Z$ in the network flow problem. How does this notion of correlation behave with respect to the graph distance between $e$ and $Z$? We now address this type of questions, and we present upper bounds that decay exponentially fast with rate controlled by the second largest eigenvalue in magnitude of the diffusion random walk naturally defined on $(V,E,W(b))$.

\subsection{Decay of correlation}
Lemma \ref{lem:comparisontheoremnetworkflow} expresses the correlation quantity for the min-cost network flow problem in terms of the Moore-Penrose pseudoinverse of the Laplacian $L(b):= D(b)-W(b)$ for the undirected weighted graph $(V,E,W(b))$. To investigate the behavior of this quantity as a function of the natural distance in the \emph{unweighted} graph $G=(V,E)$, we develop a general connection between the pseudoinverse of the Laplacian and the Green's function of the random walk with transition matrix $P(b):=D(b)^{-1}W(b)$. To the best of our knowledge, this connection --- which we present as standalone in Appendix \ref{sec:Laplacians and random walks} --- has not been previously investigated in the literature. Presently, we only state the main result on the decay of correlation for the min-cost network flow problem.

Let $n:=|V|$ be the cardinality of $V$, and for each $b\in\operatorname{Im}(A)$ let $-1\le\lambda_n(b) \le \lambda_{n-1}(b) \le \cdots \le \lambda_2(b) < \lambda_1(b) =1$ be the real eigenvalues of $P(b)$.\footnote{This characterization of eigenvalues for random walks on connected weighted graphs follows from the Perron-Frobenius theory. See \cite{lovasz1993random}, for instance.} Define $\lambda(b):=\max\{|\lambda_2(b)|, |\lambda_n(b)|\}$ and $\lambda:=\sup_{b\in\operatorname{Im}(A)} \lambda(b)$.
For each $v\in V$, let $\mathcal{N}(v):=\{w\in V: \{v,w\}\in E\}$ be the set of node neighbors of $v$ in the graph $G$. Let $d$ be the graph-theoretical distance between vertices in the graph $G$, namely, $d(u,v)$ is the length of the shortest path between vertices $u,v\in V$. Recall the definition of the localized $\ell_2$-norm from Remark \ref{rem:notation}.
The following result shows that the solution of the min-cost network flow problem satisfies a decay of correlation bound in the localized $\ell_2$-norm, with exponential rate given by $\lambda$. The proof is given at the end of Appendix \ref{sec:Laplacians and random walks}.

\begin{theorem}[Decay of correlation in the $\ell_2$-norm]
\label{thm:Decay of correlation}
Let $\varepsilon\in\mathbb{R} \rightarrow b(\varepsilon)\in \operatorname{Im}(A)$ be a differentiable function such that for any $\varepsilon\in\mathbb{R}$ we have $\frac{d b(\varepsilon)_v}{d \varepsilon} \neq 0$ if and only if $v\in Z$. Then, for any $(U,\vec{F})$ subgraph of $\vec{G}=(V,\vec E)$, we have
$$
	\sup_{\varepsilon\in\mathbb{R}} 
	\left\| \frac{d x^\star(b(\varepsilon))}{d \varepsilon} \right\|_{\vec{F}}
	\le c\,
	\frac{\lambda^{d(U,Z)}}{1-\lambda}
	\, \sup_{\varepsilon\in\mathbb{R}} 
	\left\|\frac{d b(\varepsilon)}{d \varepsilon}\right\|_Z,
$$
with $c:= \sup_{b\in\operatorname{Im}(A)}\frac{\max_{v\in U} \sqrt{2 |\mathcal{N}(v) \cap U|}}{\min_{v\in U} d(b)_v} \max_{u,v\in U} W(b)_{uv}$.
\end{theorem}

Recall that $\| \frac{d x^\star(b(\varepsilon))}{d \varepsilon} \|_{\vec{F}} \equiv \sqrt{\sum_{e\in\vec{F}} (\frac{d x^\star(b(\varepsilon))_e}{d \varepsilon})^2}$ and $\|\frac{d b(\varepsilon)}{d \varepsilon}\|_Z \equiv \sqrt{\sum_{v\in Z} (\frac{d b(\varepsilon)_v}{d \varepsilon})^2}$. Clearly, the bound in Theorem \ref{thm:Decay of correlation} controls the effect that localized perturbations that are supported on a subset of vertices $Z\subseteq V$ have on a subset of edges $\vec{F}\subseteq \vec{E}$, as a function of the distance between $\vec{F}$ and $Z$, i.e., $d(U,Z)$ (we only defined the distance among vertices, not edges). A key property --- which is essential for the results in Section \ref{sec:Scale-free algorithm} --- is that this bound does not depend on the cardinality of $\vec{F}$.



In the next section we investigate the consequences of the decay of correlation property established by Theorem \ref{thm:Decay of correlation} in the theory of local algorithms. We show that this is a fundamental property that can be used to develop scale-free algorithms for large network optimization problems.

%
%


\section{Scale-free algorithms}\label{sec:Scale-free algorithm}
Let us consider the min-cost network flow problem defined in the previous section, for a certain external flow $b\in \mathbb{R}^V$ such that $\mathbb{1}^Tb =0$. Let $Z\subseteq V$, and choose $p\in\mathbb{R}^V$ such that $\mathbb{1}^Tp =0$ and such that $p$ is supported on $Z$, namely, $p_v\neq 0$ if and only if $v\in Z$. Assume that we perturb the external flow $b$ by adding $p$.
We want to address the following question: given knowledge of the solution $x^\star(b)$ for the unperturbed problem, what is a \emph{computationally efficient} algorithm to compute the solution $x^\star(b+p)$ of the perturbed problem?
The basic idea that we aim to exploit is that the decay of correlation property established in Theorem \ref{thm:Decay of correlation} implies
that a localized perturbation of the external flow affects more the components of $x^\star(b)$ that are close to the perturbed sites. As a result, only a subset of the components of the solution around the perturbed region $Z$ needs to be updated to meet a prescribed level of error precision, yielding savings on the computational complexity.

To formalize this idea, henceforth let $\vec G'=(V',\vec E')$ be a subgraph of $\vec G=(V,\vec E)$ such that $Z\subseteq V'$. Let $G'=(V',E')$ be the undirected graph associated to $\vec G$ (see Remark \ref{rem:notation}), and assume that $G'$ is connected. Define $V'^C:=V\setminus V'$ and $\vec E'^C:=\vec E\setminus \vec E'$. We now introduce a \emph{local} algorithm to approximately compute $x^\star(b+p)$. This algorithm only updates the components of $x^\star(b)$ --- which is assumed to be known --- on the subset $\vec E'$.


\subsection{Localized projected gradient descent algorithm}

As a general-purpose algorithm for constrained convex optimization, we consider the canonical projected gradient descent algorithm. The same argument about localization that we are about to present can analogously be developed for other optimization procedures (we refer to \cite{B14} for a recent review of algorithmic procedures in large-scale optimization, and to \cite{BT97} for a book reference). Recall that a single iteration of the projected gradient descent algorithm to compute $x^\star(b)$ is the map $T_b := x\in\mathbb{R}^{\vec E}\rightarrow T_b(x)\in\mathbb{R}^{\vec E}$ defined as
$$
	T_b(x) :=
	{\arg\min} \left\{ 
	\| u - (x - \eta \nabla f(x) ) \|
	:
	u\in\mathbb{R}^{\vec E},A u = b
	\right\},
$$
where $\eta > 0$ is a given step size. Let each function $f_e$ be $\alpha$-strongly convex and $\beta$-smooth, i.e.,
$
	\alpha \le \frac{d^2 f_e(x)}{d x^2} \le \beta,
$
for each $x\in\mathbb{R}$. A classical result yields that the projected gradient descent with step size $\eta=\frac{1}{\beta}$ converges to the optimal solution of the problem, namely, 
$
	\lim_{t\rightarrow\infty} T^t_{b}(x) = x^\star(b)
$
for any starting point $x\in\mathbb{R}^{\vec E}$, where $T^t_b$ defines the $t$-th iteration of the algorithm. In the $\ell_2$-norm, the convergence rate is given by (see \cite{B14}[Theorem 3.6], for instance)
$$
	\| T^t_{b}(x) - x^\star(b) \| \le e^{- t/(2Q)} \| x - x^\star(b) \|,
$$
where $Q=\beta/\alpha$ is the so-called \emph{condition number}.
We naturally define the \emph{localized} projected gradient descent on $\vec G'$ as follows (recall from Remark \ref{rem:notation} the notation for submatrices).

\begin{definition}[Localized projected gradient descent]
Given $x\in\mathbb{R}^{\vec E}$ such that $A_{V'^C,\vec E'^C}x_{\vec E'^C} = b_{V'^C}$, the \emph{localized projected gradient descent on $\vec G'$} with step size $\eta>0$ is defined as
$$
	T'_{b}(x) :=
	\arg\min
	\left\{
	\| u - (x - \eta \nabla f(x) ) \|
	:
	u\in\mathbb{R}^{\vec E}:A u = b, u_{\vec E'^C}=x_{\vec E'^C}
	\right\}.
$$
\end{definition}
Only the components of $x$ supported on $\vec E'$ are updated by $T'_{b}$, while the components on $\vec E'^C$ stay fixed, playing the role of boundary conditions: for $e\in \vec E'^C$ we have $T'_b(x)_e = x_e$. For this reason, the map $T'_{b}$ is defined only for the points $x\in\mathbb{R}^{\vec E}$ whose coordinates outside $\vec E'$ are consistent with the constraint equations.
%
The algorithm that we propose to compute $x^\star(b+p)$ given knowledge of $x^\star(b)$ is easily described: it amounts to running for $t$ times the localized projected gradient descent on $\vec G'$ with ``frozen" boundary conditions $x^\star(b)_{\vec E'^C}$ (and step size $\eta=1/\beta$), namely,
$
	T'^t_{b(\varepsilon)}(x^\star(b)).
$
Clearly, $x^\star(b)$ satisfies the flow conservation constraints on $\vec E'^C$, by definition.


\subsection{Error analysis: bias-variance decomposition}
We now provide estimates for the error committed by the localized projected gradient descent as a function of the subgraph $\vec G'$ and the running time $t$. The key ingredient behind our estimates is the decay of correlation property for the min-cost network flow problem established in Theorem \ref{thm:Decay of correlation}.

Let us define the error committed by the localized projected gradient descent algorithm after $t \ge 1$ iterations as the vector in $\mathbb{R}^V$ given by
$$
	\operatorname{Error}(p,\vec G',t)
	:= x^\star(b+p) - T'^t_{b+p}(x^\star(b)).
$$
The analysis that we give is based on the following decomposition, that resembles the bias-variance decomposition in statistical analysis:
$
	\operatorname{Error}(p,\vec G',t)
	= \operatorname{Bias}(p,\vec G')
	+ \operatorname{Variance}(p,\vec G',t),
$
where
\begin{align*}
	\operatorname{Bias}(p,\vec G')
	:=& \ x^\star(b+p) - \lim_{t\rightarrow\infty}
	T'^t_{b+p}(x^\star(b)),\\
	\operatorname{Variance}(p,\vec G',t)
	:=& \
	\lim_{t\rightarrow\infty} T'^t_{b+p}(x^\star(b))
	- T'^t_{b+p}(x^\star(b)).
\end{align*}
The bias term is algorithm-independent --- any algorithm that converges to the optimal solution yields the same bias --- and it characterizes the error that we commit by localizing the optimization procedure per se, as a function of the subgraph $\vec G'$. On the other hand, the variance term depends on the specific choice of the algorithm that we run inside $\vec G'$.

Let define the \emph{inner boundary} of $\vec G'$ as
\begin{align*}
	\Delta(\vec G') &:=
	\{ v\in V' : \mathcal{N}(v) \cap  V'^C \neq \varnothing \}.
\end{align*}
Let $B\in\mathbb{R}^{V\times V}$ be the \emph{vertex-to-vertex adjacency matrix} of the undirected graph $G=(V,E)$, which is the symmetric matrix defined as $B_{uv}:=1$ if $\{u,v\}\in E$, $B_{uv}:=0$ otherwise.
Being real and symmetric, the matrix $B$ has $n:=|V|$ real eigenvalues which we denote by $\mu_{n} \le \mu_{n-1} \le \cdots \le \mu_2 \le \mu_1$. Let $\mu:=\max\{|\mu_2|,|\mu_{n}|\}$ be the second largest eigenvalue in magnitude of $B$.

The next theorem yields bounds for the bias and variance error terms in the $\ell_2$-norm.  The bound for the bias decays exponentially with respect to the graph-theoretical distance (i.e., the distance in the unweighted graph $G$) between the inner boundary of $\vec G'$, i.e., $\Delta(\vec G')$, and the region where the perturbation $p$ is supported, i.e., $Z\subseteq V$. The rate is governed by the eigenvalue $\mu$, the condition number $Q$, and the maximum/minimum degree of the graph.
The bound for the variance decays exponentially with respect to the running time, with rate proportional to $1/Q$. The proof of this theorem is given in Appendix \ref{app:error localized}, and the key ingredient is the decay of correlation property for the min-cost network flow problem established in Theorem \ref{thm:Decay of correlation}.

\begin{theorem}[Error localized algorithm]
\label{thm:error localized}
Let $k_-$ and $k_+$ be, respectively, the minimum and maximum degree of $G$. Let $\rho:=\frac{Qk_+}{k_-} -1 + \frac{Q}{k_-} \mu$. If $\rho< 1$, then
\begin{align*}
	\| \operatorname{Bias}(p,\vec G') \|
	&\le 
	\| p \| \, \gamma\,
	\frac{\rho^{d(\Delta(\vec G'),Z)}}{(1-\rho)^2} \, \mathbf{1}_{\vec G'\neq \vec G},
	&\| \operatorname{Variance}(p,\vec G',t) \|
	&\le 
	\| p \|\, c\,
	\frac{e^{- t/(2Q)}}{1-\rho},
\end{align*}
with $\gamma := c
	\left( 1
	+
	c\sqrt{k_+-1}
	\right)$ and $c:=\frac{\sqrt{2k_+}}{k_-} Q$.
The bound for total error committed by the algorithm follows by the triangle inequality for the $\ell_2$-norm, namely,
$$
	\| \operatorname{Error}(p,\vec G',t) \|
	\le \| \operatorname{Bias}(p,\vec G') \|
	+ \| \operatorname{Variance}(p,\vec G',t) \|.
$$
\end{theorem}

Note that the constants appearing in the bounds in Theorem \ref{thm:error localized} do not depend on the choice of the subgraph $\vec G'$ of $\vec G$, but depend only on $\mu$, $Q$, $k_+$, and $k_-$ (a more refined analysis can yield better constants that do depend on the choice of $\vec G'$, but we do not need them for our purposes). In particular, the same constants apply for the analysis of the \emph{global} algorithm, i.e., the projected gradient descent applied to the entire graph $\vec G$. In this case, the bias term clearly equals $0$, so that the error is equivalent to the variance (hence the indicator function $\mathbf{1}_{\vec G'\neq \vec G}$ in Theorem \ref{thm:error localized}).

Analogously to what happens in the statistical setting, in the next section we show that the bias introduced by the localization procedure can be exploited to lower the computational complexity that is associated to the variance term. This is the key idea that allows us to prove dimension-free computational complexity for the localized projected gradient descent algorithm.

\subsection{Dimension-free computational complexity}
The error estimates established in Theorem \ref{thm:error localized} allow to prove that the localized projected gradient descent is scale-free, in the sense that it is guaranteed to meet a prescribed level of error accuracy $\varepsilon >0$ with a computational complexity that does not depend on the dimension of the network $\vec G$.

To illustrate this fact, let $G=(V,E)$ be a $k$-regular graph such that the second largest eigenvalue in magnitude of its vertex-to-vertex adjacency matrix is bounded away from $k$ as a function of the dimension of $G$: namely, $\mu \le \gamma < k$, where $\gamma$ is a universal constant that does not depend on the size $|V|$, nor on the size $|E|$. This is the same as saying that $G$ comes from a family of $k$-regular expander graphs \citep{Hoory06expandergraphs}. Define $\vec G=(V,\vec E)$ by assigning an arbitrary orientation to the edges of $G$.
Assume that the following holds: $\rho=Q-1+\frac{Q}{k}\mu < 1$, where recall that $Q=\beta/\alpha$ is the condition number. For the sake of simplicity, we introduce a collection of subgraphs of $\vec G$ that are centered on a given vertex and are parametrized by their radii. Fix a vertex $v\in V$.
Let $V_r:=\{w\in V: d(v,w)\le r\}$ denote the ball of radius $r>0$ around vertex $v\in V$, and let $\vec G_r:=(V_r,\vec E_r)$ be the subgraph of $\vec G$ that has vertex set $V_r$, and induced edge set $\vec E_r$. Let $r_{\text{max}}:=\max\{d(v,w):w\in V\}$. Consider a perturbation vector $p\in\mathbb{R}^V$ that is supported on $Z:=V_z$, for a fixed $z>0$. If we run the localized algorithm on $\vec G_r$, with $r> z$, for $t$ time steps, then Theorem \ref{thm:error localized} yields the following estimate (here $d(\Delta(\vec G_r),Z)=r-z$ and $\mathbf{1}_{\vec G_r\neq \vec G}=\mathbf{1}_{r< r_{\text{max}}}$):
\begin{align*}
	\left\| \operatorname{Error}(p,\vec G_r,t) \right\|
	\le  
	\| p \| \, \nu_{\text{bias}}\, e^{- \xi_{\text{bias}} r} \, \mathbf{1}_{r< r_{\text{max}}}
	+
	\| p \|\, \nu_{\text{var}} \, e^{- \xi_{\text{var}} t},
\end{align*}
with $\nu_{\text{bias}}:=\frac{\gamma}{(1-\rho)^2\rho^z}$, $\xi_{\text{bias}}:=\log\frac{1}{\rho}>0$, $\nu_{\text{var}}:=\frac{c}{1-\rho}$, and $\xi_{\text{var}}:=\frac{1}{2Q}>0$, where $\gamma := c
	( 1
	+
	c\sqrt{k-1})$ and $c:=\sqrt{2} Q/\sqrt{k}$. 

Let $\kappa(\vec G_r,t)$ be the computational complexity required to run the localized projected gradient descent algorithm on $\vec G_r$ for $t$ time steps. A rough estimate for the asymptotic behavior of $\kappa(\vec G_r,t)$ is easily derived as follows (more refined estimates can be made, but we do not need them to make our point). If $A_r:=A_{V_r,\vec E_r}$ denotes the vertex-to-edge adjacency matrix associated to $\vec G_r$, and $f_r := \sum_{e\in \vec E_r} f_e$ is the restriction of the cost function $f$ to the edges in $\vec G_r$, it is easy to check that a single iteration of the localized projected gradient descent algorithm on $\vec G_r$ reads
\begin{align*}
	T^{(r)}_{b+p}(x)_{\vec E_r} 
	&=
	(I\!-\!A_r^T(A_rA_r^T)^{+}A_r)(x_{\vec E_r}\!-\!\eta \nabla f_r(x_{\vec E_r}))
	\!+\! A_r^T(A_rA_r^T)^{+}(b_{V_r} \!+\! p_{V_r} \!-\! A_{V_r,\vec E_r^C} x_{\vec E_r^C} ),\\
	T^{(r)}_{b+p}(x)_{\vec E_r^C} 
	&=
	x_{\vec E_r^C},
\end{align*}
for any $x\in\mathbb{R}^{\vec E}$ such that $A_{V_r^C,\vec E_r^C}x_{\vec E_r^C} = b_{V_r^C}$. The exact computation of the matrix $(A_rA_r^T)^{+}$ has an asymptotic complexity that scales like $O(|V_r|^\omega)$ as a function of $r$, where $\omega > 2$ is the matrix multiplication constant.\footnote{The same rationale behind the argument that we make applies if we consider \emph{approximate} algorithms that are taylor-made to take advantage of the Laplacian structure of the matrix $A_rA_r^T$ and yield much better computational complexity to $\delta$-compute $(A_rA_r^T)^+$, of the order of $\tilde O(|\vec{E_r}|\log |V_r| \log(1/\delta))$, see \cite{KoutisMP11}.} As each matrix-vector multiplication has a cost of $O(|V_r|^2)$, then $\kappa(\vec G_r,t)$ scales like $O(|V_r|^\omega + |V_r|^2 t)$. For the sake of simplicity, consider $O(|V_r|^\omega t)$. To estimate the complexity of the local algorithm, we need to bound the growth of $|V_r|$ as a function of $r$. In a $k$-regular graph, we clearly have $|V_r| \le k^r$ (which is realistic for expander graphs, as they are locally tree-like) so that $\kappa(\vec G_{r},t)$ grows at most as
$
	O(e^{(\omega \log k)r}t).
$

We are now in the position to appreciate the computational savings that the localized algorithm offers over the global algorithm (i.e., the projected gradient descent on $\vec G$). Assume that $\vec G$ is an infinite network with $r_{\text{max}}=\infty$. In this case, the computational complexity of the global algorithm is clearly infinity, as the global algorithm updates the components of the solution at every edge of the entire network. On the other hand, the complexity of the localized projected gradient descent algorithm is finite. This can be seen if we seek, for example, for the minimal radius $r$ and time $t$ such that
$
	\nu_{\text{bias}}\, e^{- \xi_{\text{bias}} r} 
	\le \frac{\varepsilon}{2}
$
and 
$
	\nu_{\text{var}} \, e^{- \xi_{\text{var}} t} 
	\le \frac{\varepsilon}{2}.
$
Clearly, these constraints guarantee that $\| \operatorname{Error}(p,\vec G_r,t) \| \le \varepsilon$, and it is easy to see that both the minimal $t$ and the minimal $r$ that satisfy the above inequalities scale like $O(\log (\|p\|/\varepsilon))$, so that the complexity of the localized algorithm
scales like
$
	O((\|p\|/\varepsilon)^{\omega \log k} \log(\|p\|/\varepsilon)),
$
where the constants involved do not dependent of the dimension of the graph $\vec G$, but depend only on $\mu$, $Q$, and $k$.

The decay of correlation property exhibited by the min-cost network flow problem allowed us to show that the bias introduced by localizing the optimization problem to a subgraph $\vec G_r$ saves us from the computational complexity associated to the variance term, which corresponds to running the gradient descent algorithm on $\vec G_r$ for $t$ time steps. In fact, a finer analysis shows that one can exploit the bias-variance trade-off to optimally tune the algorithm, i.e., to find a radius $r(\varepsilon)$ and a time $t(\varepsilon)$ that minimize the computational complexity $\kappa(\vec G_{r(\varepsilon)},t(\varepsilon))$ which is required to reach the prescribed level of error accuracy $\varepsilon$. These ideas suggest a general framework to study the trade-off between statistical accuracy and computational complexity for local algorithms in optimization.
\acks{We would like to thank Rasmus Kyng and Sushant Sachdeva for useful discussions.}

\bibliography{bib}

\appendix


\section{Hadamard's global inverse function theorem}
\label{sec:Hadamard global inverse theorem}

Recall that a function from $\mathbb{R}^m$ to $\mathbb{R}^m$ is said to be \emph{$C^k$} if it has continuous derivatives up to order $k$. A function is said to be a \emph{$C^k$ diffeomorphism} if it is $C^k$, bijective, and its inverse is also $C^k$. The following important result characterizes when a $C^k$ function is a $C^k$ diffeomorphism.

\begin{theorem}[Hadamard's global inverse function theorem]
\label{thm:Hadamard global inverse theorem}
Let $\Psi$ be a $C^k$ function from $\mathbb{R}^m$ to $\mathbb{R}^m$. Then, $f$ is a $C^k$ diffeomorphism if and only if the following two conditions hold:
\begin{enumerate}
\item The determinant of the differential of $\Psi$ is different from zero at any point, namely,
$|\frac{d}{dz}\Psi (z)| \neq 0$ for any $z\in\mathbb{R}^m$.
\item The function $\Psi$ is \emph{norm coercive}, namely, for any sequence of points $z_1,z_2,\ldots \in \mathbb{R}^m$ with $\| z_k \|\rightarrow\infty$ it holds $\| \Psi(z_k) \|\rightarrow\infty$ (for any choice of the norm $\| \cdot \|$, as norms are equivalent in finite dimension).
\end{enumerate}
\end{theorem}

\begin{proof}
See \cite{WD72}[Corollary of Lemma 2], for instance.
\end{proof}

The following result, which is the backbone behind the proof of Theorem \ref{thm:comparisontheorem}, comes as a corollary to the previous theorem.

\begin{lemma}[Diffeomorphism for Lagrangian multipliers map]
\label{lem:Hadamard global inverse theorem}
Let $f:\mathbb{R}^n\rightarrow \mathbb{R}$ be a strongly convex function, twice continuously differentiable.
Let $A\in\mathbb{R}^{p\times n}$ be a given matrix. Define the function $\Phi$ from $\mathbb{R}^n\times\mathbb{R}^p$ to $\mathbb{R}^n\times\mathbb{R}^p$ as
$$
	\Phi
	(
	x,
	\nu 
	)
	:=
	\left( 
	\begin{array}{c}
	\nabla f(x) + A^T\nu \\
	A x 
	\end{array} \right),
$$
for any $x\in\mathbb{R}^n$, $\nu\in\mathbb{R}^p$.
Then, the restriction of the function $\Phi$ to $\mathbb{R}^n\times \operatorname{Im}(A)$ is a $C^1$ diffeomorphism.
\end{lemma}

\begin{proof}
Let us interpret $\Phi$ as the representation of a transformation $\mathcal{T}$ in the standard basis of $\mathbb{R}^n\times \mathbb{R}^p$. Recall the orthogonal decomposition $\mathbb{R}^p = \operatorname{Im}(A) \oplus \operatorname{Ker}(A^T)$. Let the vectors $u_1,\ldots,u_r\in\mathbb{R}^{p}$ form an orthogonal basis for $\operatorname{Im}(A)$, where $r$ is the rank of $A$, and let the vectors $v_1,\ldots,v_{p-r}\in\mathbb{R}^{p}$ form an orthogonal basis for $\operatorname{Ker}(A^T)$. Define the orthogonal matrix $Z=[u_1,\ldots,u_r,z_1,\ldots,z_{p-r}]$, which represents a change of basis in $\mathbb{R}^p$. As we have
$$
	\Phi
	(
	x,
	\nu 
	)
	=
	\left( 
	\begin{array}{c}
	\nabla f(x) + A^T Z Z^T \nu \\
	Z Z^T A x 
	\end{array} \right),
$$
then the transformation $\mathcal{T}$ is represented in the standard basis for $\mathbb{R}^n$ and in the basis $Z$ for $\mathbb{R}^p$ by the following map $\widetilde\Phi$
$$
	\widetilde\Phi
	(
	x,
	\tilde\nu 
	)
	:=
	\left( 
	\begin{array}{c}
	\nabla f(x) + \widetilde A^T \tilde\nu \\
	\widetilde A x 
	\end{array} \right),
$$
where $\widetilde A := Z^T A$. In fact,
$$
	\widetilde\Phi(x,Z^T\nu) = 
	\left( \begin{array}{cc}
	I & \mathbb{O} \\
	\mathbb{O}^T & Z^T 
	\end{array} \right)
	\Phi(x,\nu),
$$
where $I\in \mathbb{R}^{n\times n}$ is the identity matrix, and $\mathbb{O}\in \mathbb{R}^{n\times p}$ is the all-zero matrix. As
$$
	A^TZ = [A^Tu_1,\ldots,A^Tu_r,A^Tv_1,\ldots,A^Tv_{p-r}]
	= [A^Tu_1,\ldots,A^Tu_r,\mathbb{O}_{n\times (p-r)}],
$$
we have
$$
	\widetilde A = (A^TZ)^T 
	= \left[
	\begin{array}{c}
	B\\
	\mathbb{O}_{(p-r)\times n}
	\end{array} \right],
$$
where $B:=[u_1,\ldots,u_r]^T A\in\mathbb{R}^{r\times n}$. Therefore, the restriction of the transformation $\mathcal{T}$ to the invariant subspace $\mathbb{R}^n\times\operatorname{Im}(A)$ is represented in the standard basis for $\mathbb{R}^n$ and in the basis $\{u_1,\ldots,u_r\}$ for $\operatorname{Im}(A)$ by the following map
$$
	\Psi
	(
	x,
	\xi 
	)
	:=
	\left( 
	\begin{array}{c}
	\nabla f(x) + B^T \xi \\
	B x 
	\end{array} \right).
$$
As the function $f$ is twice continuously differentiable, clearly the function $\Psi$ is continuously differentiable, i.e., $C^1$. We now check that the two conditions of Theorem \ref{thm:Hadamard global inverse theorem} are satisfied.

The differential of $\Psi$ evaluated at $(x,\xi)\in\mathbb{R}^n\times\mathbb{R}^r$ is given by the Jacobian matrix
$$
	J(x,\xi)
	:= 
	\left( \begin{array}{cc}
	\nabla^2 f(x) & B^T \\
	B & \mathbb{O} 
	\end{array} \right).
$$
As $f$ is strongly convex, $\nabla^2 f(x)$ is positive definite so invertible. Then, the determinant of the Jacobian can be expressed as $|J(x,\xi)|= |\nabla^2 f(x)||-B\nabla^2 f(x)^{-1}B^T|$. As $B$ has full row rank by definition, $B\nabla^2 f(x)^{-1}B^T$ is positive definite and we clearly have $|J(x,\xi)| \neq 0$.


To prove that the function $\Psi$ is norm coercive, let us choose $\| \cdot \|$ to be the Euclidean norm and consider a sequence $(x_1,\xi_1),(x_2,\xi_2),\ldots \in \mathbb{R}^n\times\mathbb{R}^r$ with $\| (x_k,\xi_k) \|\rightarrow\infty$. As for any $x\in\mathbb{R}^{n}, \xi\in\mathbb{R}^{r}$ we have $\| (x,\xi) \|^2 = \|x\|^2 + \|\xi\|^2$,
clearly for the sequence to go to infinity one of the following two cases must happen:
\begin{enumerate}[(a)]
\item $\|x_k\|\rightarrow\infty$;
\item $\|x_k\|\le c$ for some $c< \infty$, $\|\xi_k\|\rightarrow\infty$.
\end{enumerate}
Before we consider these two cases separately, let us note that, for any $x\in\mathbb{R}^{n}, \xi\in\mathbb{R}^{r}$,
\begin{align}
	\| \Psi(x,\xi) \|^2 = \| \nabla f(x) + B^T\xi \|^2 + \| Bx \|^2.
	\label{bound}
\end{align}
Let $\alpha > 0$ be the strong convexity parameter, and recall the following definition of strong convexity, for any $x,y\in\mathbb{R}^{n}$,
\begin{align}
	(\nabla f(x)-\nabla f(y))^T(x-y) \ge \alpha \| x-y \|^2.
	\label{strong convexity}
\end{align}

\begin{enumerate}[(a)]
\item
Assume $\|x_k\|\rightarrow\infty$.
Let $P_{\parallel}$ be the projection operator on $\operatorname{Im}(B^T)$, i.e., $P_{\parallel}:=B^T(BB^T)^{-1}B$, and let $P_{\perp}=I-P_{\parallel}$ be the projection operator on $\operatorname{Ker}(B)$, the orthogonal complement of $\operatorname{Im}(B^T)$. As for any $x\in\mathbb{R}^n$ we have the decomposition $x=P_{\parallel}x + P_{\perp}x$ with $(P_{\parallel}x)^TP_{\perp}x =0$, clearly $\|x\|^2=\|P_\parallel x\|^2+\|P_\perp x\|^2$. So, the condition $\|x_k\|\rightarrow\infty$ holds only if one of the two cases happens:
\begin{enumerate}[(i)]
\item $\|P_\parallel x_k\|\rightarrow\infty$;
\item $\|P_\parallel x_k\|\le c$ for some $c< \infty$, $\|P_\perp x_k\|\rightarrow\infty$.
\end{enumerate}

Consider the case (i) first. Let $x\in\mathbb{R}^n$ so that $P_\parallel x \neq \mathbb{O}$. As $BP_\perp = \mathbb{O}$, from \eqref{bound} we have, for any $\xi\in\mathbb{R}^{r}$,
$$
	\| \Psi(x,\xi) \|^2 \ge \| Bx \|^2 = \| B P_\parallel x \|^2
	\ge \min_{y\in\mathbb{R}^{n} : y\in \operatorname{Im}(B^T), y\neq \mathbb{O}} \frac{y^TB^TBy}{\|y\|^2} \|P_\parallel x\|^2
	= \lambda \|P_\parallel x\|^2,
$$
where $\lambda$ is the minimum eigenvalue of $B^TB$ among those corresponding to the eigenvectors spanning the subspace $\operatorname{Im}(B^T)$. Clearly, if $\lambda\neq 0$ (notice $\lambda \ge 0$ by definition) then the above yields that $\| \Psi(x_k,\xi_k) \| \rightarrow \infty$ whenever $\|P_\parallel x_k\|\rightarrow\infty$. To prove this, assume by contradiction that $\lambda=0$. Then, there exists $y\in\mathbb{R}^{n}$ satisfying $y\in \operatorname{Im}(B^T), y\neq \mathbb{O}$, such that $B^TBy=\lambda y= \mathbb{O}$. As $B^T$ has full column rank by assumption, the latter is equivalent to $By=\mathbb{O}$ so that $P_\perp y = y\neq \mathbb{O}$, which contradicts the hypothesis that $y\in \operatorname{Im}(B^T)$.

Consider now the case (ii). Decomposing the gradient on $\operatorname{Im}(B^T)$ and its orthogonal subspace, from \eqref{bound} we have, for any $x\in\mathbb{R}^{n}, \xi\in\mathbb{R}^{r}$,
\begin{align*}
	\| \Psi(x,\xi) \|^2 \ge \| P_\perp\nabla f(x) + P_\parallel\nabla f(x) + B^T\xi \|^2
	= \| P_\perp\nabla f(x) \|^2 + \| P_\parallel\nabla f(x) + B^T\xi \|^2,
\end{align*}
so that $\| \Psi(x,\xi) \| \ge \| P_\perp\nabla f(x) \|$. 
Choosing $y=P_\parallel x$ in \eqref{strong convexity} we have
$$
	(P_\perp\nabla f(x)-P_\perp\nabla f(P_\parallel x))^TP_\perp x=
	(\nabla f(x)-\nabla f(P_\parallel x))^TP_\perp x \ge \alpha \| P_\perp x \|^2,
$$
and applying Cauchy-Schwarz we get, for any $x$ such that $P_\perp x \neq \mathbb{O}$,
$$
	\|P_\perp\nabla f(x)\|
	\ge \alpha \| P_\perp x \| - \| P_\perp\nabla f(P_\parallel x) \|.
$$
By assumption $f$ is twice continuously differentiable, so $\nabla f$ is continuous and it stays bounded on a bounded domain. Hence, we can conclude that $\| \Psi(x_k,\xi_k) \| \rightarrow \infty$ if $\|P_\perp x_k\|\rightarrow\infty$ with $(P_\parallel x_k)_{k\ge 1}$ bounded.

%
%

\item Assume $\|\xi_k\|\rightarrow\infty$ and $(x_k)_{k\ge 1}$ bounded. Notice that for any $\xi\in\mathbb{R}^{r}$, $\xi \neq \mathbb{O}$, we have
$$
	\| B^T\xi \|^2 = \frac{\xi^TBB^T\xi}{\|\xi\|^2} \|\xi\|^2
	\ge \min_{y\in\mathbb{R}^{r} : y\neq \mathbb{O}} \frac{y^TBB^Ty}{\|y\|^2} \|\xi\|^2
	= \lambda_{\text{min}} \|\xi\|^2,
$$
where $\lambda_{\text{min}}$ is the minimum eigenvalue of $BB^T$, which is strictly positive as $BB^T$ is positive definite by the assumption that $B$ has full row rank. From \eqref{bound} we have
$$
	\| \Psi(x,\xi) \|
	\ge \| \nabla f(x) + B^T\xi \|
	\ge \| B^T\xi \| - \| \nabla f(x) \|
	\ge \sqrt{\lambda_{\text{min}}} \|\xi\| - \| \nabla f(x) \|,
$$
that, by continuity of $\nabla f$, shows that $\| \Psi(x_k,\xi_k) \| \rightarrow \infty$ if $\|\xi_k\|\rightarrow\infty$ and $(x_k)_{k\ge 1}$ is bounded.
\end{enumerate}
\end{proof}


\section{Laplacians and random walks}\label{sec:Laplacians and random walks}
Let $G=(V,E,W)$ be a simple (i.e., no self-loops, and no multiple edges), connected, undirected, weighted graph, where to each edge $\{v,w\}\in E$ is associated a non-negative weight $W_{vw}=W_{wv}>0$, and $W_{vw}=0$ if $\{v,w\}\not\in E$. Let $D$ be a diagonal matrix with entries $d_v=D_{vv}=\sum_{w\in V} W_{vw}$ for each $v\in V$.
For each vertex $v\in V$,
let $\mathcal{N}(v):=\{w\in V: \{v,w\}\in E\}$ be the set of node neighbors of $v$.
In this section we establish several connections between the graph Laplacian $L:=D-W$ and the random walk $X:=(X_t)_{t\ge 0}$ with transition matrix $P:=D^{-1}W$.
Henceforth, for each $v\in V$, let $\mathbf{P}_{v}$ be the law of a time homogeneous Markov chain $X_{0},X_{1},X_2,\ldots$ on $V$ with transition matrix $P$ and initial condition $X_0=v$. Analogously, denote by $\mathbf{E}_{v}$ the expectation with respect to this law. The hitting time to the site $v\in V$ is defined as
$
	T_v := \inf\{t\ge0 : X_t = v\}.
$
Let $\pi$ be the unique stationary distribution of the random walk, namely, $\pi^TP = \pi^T$. By substitution it is easy to check that $\pi_v:=\frac{d_v}{\sum_{v\in V} d_v}$ for each $v\in V$. We adopt the notation $e_v\in\mathbb{R}^V$ to denote the vector whose only non-zero component equals $1$ and corresponds to the entry associated to $v\in V$.

\subsection{Restricted Laplacians and killed random walks}
\label{sec:Restricted Laplacians and killed random walks}
The connection between Laplacians and random walks that we present in Section \ref{sec:Pseudoinverse of graph Laplacians and Green's function of random walks} below is established by investigating \emph{restricted} Laplacians and \emph{killed} random walks. Throughout this section, let $\bar z\in V$ be fixed, and define $\bar W$ and $\bar D$ as the matrix obtained by removing the $\bar z$-th row and $\bar z$-th column form $W$ and $D$, respectively. Let $\bar V:=V\setminus \{\bar z\}$ and $\bar E:=E\setminus \{\{u,v\}\in E: u=z \text{ or } v=z\}$. 
Let $\bar L:=\bar D-\bar W$ be the restricted Laplacian that we obtain by removing the $\bar z$-th row and $\bar z$-th column form $L$. On the other hand, let $\bar P:=\bar D^{-1}\bar W$ be the transition matrix of the transient part of the killed random walk that is obtained from $X$ by adding a cemetery at site $\bar z$. Creating a cemetery at $\bar z$ means modifying the walk $X$ so that $\bar z$ becomes a recurrent state, i.e., once the walk is in state $\bar z$ it will go back to $\bar z$ with probably $1$. This is clearly done by replacing the $\bar z$-th row of $P$ by a row with zeros everywhere but in the $\bar z$-th coordinate, where the entry is equal to $1$. The relation between the transition matrix $\bar P$ of the killed random walk and the law of the random walk $X$ itself is made explicit in the next proposition.

\begin{proposition}\label{prop:killedrv}
For any $v,w\in \bar V$, $t\ge 0$, we have
$
	\bar P^t_{vw}
	=
	\mathbf{P}_v(X_t=w,T_{\bar z}>t).
$
\end{proposition}

\begin{proof}
We prove the statement by induction. Clearly, for any $v,w\in \bar V$, we have
$
	\mathbf{P}_v(X_0=w,T_{\bar z}>0) = \mathbf{P}_v(X_0=w)
	= \mathbf{1}_{v=w} = \bar P^0_{vw},
$
which proves the statement for $t=0$ ($\mathbf{1}_{v=w}$ is the indicator function).
Assume that the statement holds for any time $s\ge 0$ up to $t >0$. By the properties of conditional expectation, noticing that $\{T_{\bar z} > t+1\}=\{X_0\neq \bar z,\ldots,X_{t+1}\neq \bar z\}$, we have
\begin{align*}
	\mathbf{P}_v(X_{t+1}=w,T_{\bar z}>t+1)
	&= \mathbf{E}_v [\mathbf{P}_v(X_{t+1}=w,T_{\bar z}>t+1 | X_0,\ldots,X_t)]\\
	&= \mathbf{E}_v [\mathbf{1}_{\{X_0\neq \bar z,\ldots,X_{t}\neq \bar z\}}\mathbf{P}_v(X_{t+1}=w,X_{t+1}\neq \bar z | X_0,\ldots,X_t)].
\end{align*}
for any $v,w\in \bar V$. By the Markov property, on the event $\{X_{t}\neq \bar z\}$, we have
$$	
	\mathbf{P}_v(X_{t+1}=w,X_{t+1}\neq \bar z | X_0,\ldots,X_t) 
	= \mathbf{P}_{X_t}(X_{1}=w,X_{1}\neq \bar z)
	= \bar P_{X_tw},
$$
so that by the induction hypothesis we have
\begin{align*}
	\mathbf{P}_v(X_{t+1}=w,T_{\bar z}>t+1)
	&= \mathbf{E}_v [\mathbf{1}_{\{T_{\bar z} > t\}} \bar P_{X_tw}]
	= \sum_{u\in V\setminus\{\bar z\}} 
	\mathbf{P}_v(X_t=u,T_{\bar z}>t) \bar P_{uw}
	= \bar P^{t+1}_{vw},
\end{align*}
which proves the statement for $t+1$.
\end{proof}

The following proposition relates the inverse of the reduced Laplacian $\bar L$ with the Green function of the killed random walk, namely, the function $(u,w)\in V^2 \rightarrow \sum_{t=0}^\infty \bar P^t_{uv}$, with the hitting times of the original random walk $X$.

\begin{proposition}\label{prop:reducedlapandgreenfunction}
For each $v,w\in \bar V$, we have
\begin{align*}
	\bar L^{-1}_{vw} 
	= \frac{1}{d_w} \sum_{t = 0}^\infty \bar P^t_{vw}
	= \bar L^{-1}_{ww}\mathbf{P}_v(T_w<T_{\bar z}),
	\qquad
	\bar L^{-1}_{ww}
	= \frac{1}{d_w} \mathbf{E}_w\left[\sum_{t = 0}^{T_{\bar z}} \mathbf{1}_{X_t=w} \right].
\end{align*}
\end{proposition}

\begin{proof}
Let us first assume that $\bar G$ is connected. The matrix $\bar P$ is sub-stochastic as, clearly, if $v \not\in \mathcal{N}(\bar z)$ then $\sum_{w\in V} \bar P_{vw} = 1$, while if $v \in \mathcal{N}(\bar z)$ then $\sum_{w\in V} \bar P_{vw} < 1$. Then $\bar P$ is irreducible (in the sense of Markov chains, i.e., for each $v,w\in \bar V$ there exists $t$ to that $\bar P^t_{vw}\neq 0$) and the spectral radius of $\bar P$ is strictly less than $1$ (see Corollary 6.2.28 in \cite{Horn:1985:MA:5509}, for instance), so that the Neumann series $\sum_{t=0}^\infty \bar P^t$ converges.
The Neumann series expansion for $\bar L^{-1}$ yields
$$
	\bar L^{-1}
	= \sum_{t=0}^\infty (I-\bar D^{-1} \bar L)^t \bar D^{-1}
	= \sum_{t=0}^\infty \bar P^t \bar D^{-1},
$$
or, entry-wise, $\bar L^{-1}_{vw} = \frac{1}{d_w} \sum_{t = 0}^\infty \bar P^t_{vw}$. As $\bar P^t_{vw}=\mathbf{P}_v(X_t=w,T_{\bar z}>t)$ by Proposition \ref{prop:killedrv}, by the Monotone convergence theorem we can take the summation inside the expectation and get
$$
	\sum_{t = 0}^\infty \bar P^t_{vw} 
	= \sum_{t = 0}^\infty \mathbf{E}_v[\mathbf{1}_{X_t=w}\mathbf{1}_{T_{\bar z}>t}]
	= \mathbf{E}_v
	\left[\sum_{t = 0}^{T_{\bar z}-1} \mathbf{1}_{X_t=w}\right]
	= \mathbf{E}_v\left[\sum_{t = 0}^{T_{\bar z}} \mathbf{1}_{X_t=w}\right],
$$
where in the last step we used that $X_{T_{\bar z}}=\bar z$ and $\bar z\neq w$. Recall that if $S$ is a stopping time for the Markov chain $X:=X_0,X_1,X_2,\ldots$, then by the strong Markov property we know that, conditionally on $\{S<\infty\}$ and $\{X_S=w\}$, the chain $X_S,X_{S+1},X_{S+2},\ldots$ has the same law as a time-homogeneous Markov chain $Y:=Y_{0},Y_1,Y_2,\ldots$ with transition matrix $P$ and initial condition $Y_0=w$, and $Y$ is independent of $X_0,\ldots,X_S$. The hitting times $T_w$ and $T_{\bar z}$ are two stopping times for $X$, and so is their minimum $S:=T_w\wedge T_{\bar z}$. As either $X_S=w$ or $X_S=\bar z$, we have
$$
	\mathbf{E}_v\left[\sum_{t = 0}^{T_{\bar z}} \mathbf{1}_{X_t=w}\right]
	= \mathbf{E}_v\left[\sum_{t = 0}^{T_{\bar z}} \mathbf{1}_{X_t=w} \Bigg\vert X_S=w\right]
	\mathbf{P}_v(X_S=w),
$$
where we used that, conditionally on $\{X_S=\bar z\}=\{T_w > T_{\bar z}\}$, clearly $\sum_{t = 0}^{T_{\bar z}} \mathbf{1}_{X_t=w} = 0$. Conditionally on $\{X_S=w\}=\{T_w < T_{\bar z}\}=\{S=T_w\}$, we have $T_{\bar z}=S + \inf\{t\ge 0:X_{S+t}=\bar z\}$, and the strong Markov property yields (note that the event $\{S<\infty\}$ has probability one from any starting point, as the graph $G$ is connected by assumption so that the Markov chain will almost surely eventually hit either $w$ or $\bar z$)
$$
	\mathbf{E}_v\left[\sum_{t = 0}^{T_{\bar z}} \mathbf{1}_{X_t=w} \Bigg\vert X_S=w\right]
	=
	\mathbf{E}_v\left[\sum_{t = 0}^{\inf\{t\ge 0:X_{S+t}=\bar z\}} 
	\mathbf{1}_{X_{S+t}=w} \Bigg\vert X_S=w\right]
	= 
	\mathbf{E}_w\left[\sum_{t = 0}^{T_{\bar z}} \mathbf{1}_{X_t=w} \right].
$$
Putting everything together we have
$
	\bar L^{-1}_{vw}
	= \frac{1}{d_w} \mathbf{E}_w[\sum_{t = 0}^{T_{\bar z}} \mathbf{1}_{X_t=w} ]
	\mathbf{P}_v(T_w<T_{\bar z}).
$
As $\mathbf{P}_w(T_w<T_{\bar z})=1$, clearly
$
	\bar L^{-1}_{ww}
	= \frac{1}{d_w} \mathbf{E}_w[\sum_{t = 0}^{T_{\bar z}} \mathbf{1}_{X_t=w} ]
$
so that
$
	\bar L^{-1}_{vw}
	= \bar L^{-1}_{ww}
	\mathbf{P}_v(T_w<T_{\bar z}).
$
The argument just presented extends easily to the case when $\bar G$ is not connected. In fact, in this case the matrix $\bar P$ has a block structure, where each block corresponds to a connected component and to a sub-stochastic submatrix, so that the argument above can be applied to each block separately.
\end{proof}

The following result relates the inverse of the reduced Laplacian $\bar L$ with the pseudoinverse of the Laplacian $L$, which we denote by $L^+$. It is proved in \cite{FPRS07}[Appendix B].


\begin{proposition}\label{prop:connectionlaplacians}
For any $v,w\in \bar V$, we have
$
	\bar L^{-1}_{vw}
	= (e_{v}-e_{\bar z})^T 
	L^{+} (e_{w}-e_{\bar z}).
$
\end{proposition}

Proposition \ref{prop:reducedlapandgreenfunction} and Proposition \ref{prop:connectionlaplacians} allow us to relate the quantity $L^+$ to the \emph{difference} of the Green's function of the random walk, as we discuss next.

\subsection{Pseudoinverse of graph Laplacians and Green's function of random walks}
\label{sec:Pseudoinverse of graph Laplacians and Green's function of random walks}
We now relate the Moore-Penrose pseudoinverse of the Laplacian $L:=D-W$ with the Green's function
$
	(u,v)\in V^2
	\rightarrow
	\sum_{t=0}^\infty P^t_{uv} =
	\mathbf{E}_u[
	\sum_{t = 0}^\infty 
	\mathbf{1}_{X_t=v}]
$
of the random walk, which represents the expected number of times the Markov chain $X$ visits site $v$ when it starts from site $u$.
Notice that as the graph $G$ is finite and connected, then the Markov chain $X$ is recurrent and the Green's function itself equals infinity for any $u,v\in V$. In fact, the following result involves \emph{differences} of the Green's function, not the Green's function itself. To the best of our knowledge, this connection --- which represents the key result that will allow us to bound functions of $L^{+}$ by spectral properties of $P$ --- has not been previously investigated in the literature.\footnote{Notice that the Green's function associated with the pseudoinverse of (discrete) Laplacians \citep{Chung2000191} differs from the Green's function of random walks that we presently consider in this paper.}

\begin{lemma}
\label{lem:Laplacians and random walks}
For any $u,v,w,z\in V$, we have
\begin{align*}
	(e_{u}-e_{v})^T L^{+} (e_{w}-e_{z})
	&=
	\sum_{t=0}^\infty (e_{u}-e_{v})^T P^{t} 
	\left( \frac{e_{w}}{d_w}-\frac{e_{z}}{d_z} \right),
\end{align*}
and the same formulas hold if we swap the role of $u\leftrightarrow w$ and $v\leftrightarrow z$.
\end{lemma}

\begin{proof}
Using first Proposition \ref{prop:connectionlaplacians} and then Proposition \ref{prop:reducedlapandgreenfunction} we obtain, for any $u,v,w,z\in V$ (choose $\bar z$ to be $z$ in Section \ref{sec:Restricted Laplacians and killed random walks}),
\begin{align*}
	(e_{u}-e_{v})^T L^{+} (e_{w}-e_{z})
	&= 
	(e_{u}-e_{z})^T L^{+} (e_{w}-e_{z})
	-(e_{v}-e_{z})^T L^{+} (e_{w}-e_{z})
	=\bar L^{-1}_{uw} - \bar L^{-1}_{vw}
	\\
	&= (e_{w}-e_{z})^T L^{+} (e_{w}-e_{z})
	\left\{
	\mathbf{P}_u(T_w<T_z) -
	\mathbf{P}_v(T_w<T_z)
	\right\}.
\end{align*}
From (3.27) in the proof of Proposition 3.10 in Chapter 3 in \cite{aldous-fill-2014}, upon identifying $v\rightarrow u,x\rightarrow v,v_0\rightarrow w,a\rightarrow z$, we immediately have the following relation between the difference of potentials and hitting times of the random walk $X$:
\begin{align*}
	\mathbf{P}_u(T_w<T_z) -
	\mathbf{P}_v(T_w<T_z)
	=
	\pi_{w}\mathbf{P}_{w}(T_z<T^+_w)
	\left\{\mathbf{E}_uT_z - \mathbf{E}_vT_z + \mathbf{E}_vT_w- \mathbf{E}_uT_w\right\},
\end{align*}
where $\pi_v:=\frac{d_v}{\sum_{v\in V} d_v}$ is the $v$-th component of the stationary distribution of the random walk $X$, and $
	T^+_v := \inf\{t\ge 1 : X_t = v\}.
$
From Corollary 8 in Chapter 2 in \cite{aldous-fill-2014}, we have
$$
	\pi_{w}\mathbf{P}_{w}(T_z<T^+_w)
	= 
	\begin{cases}
		\frac{1}{\mathbf{E}_wT_z+\mathbf{E}_zT_w} &\text{if } w\neq z,\\
		\pi_w &\text{if } w= z,
	\end{cases}
$$
and we recall the connection between commute times and effective resistance (see, for example, Corollary 3.11 in \cite{aldous-fill-2014}):
$$
	\mathbf{E}_wT_z+\mathbf{E}_zT_w
	= (e_{w}-e_{z})^T L^{+} (e_{w}-e_{z}) \sum_{v\in V} d_v.
$$
Lemma 3.3 in \cite{Friedrich:2010yu} yields
$$
	\mathbf{E}_uT_z - \mathbf{E}_vT_z
	= \frac{1}{\pi_z} \sum_{t=0}^\infty (P^t_{vz} - P^t_{uz}),
	\qquad
	\mathbf{E}_uT_w - \mathbf{E}_vT_w
	= \frac{1}{\pi_w} \sum_{t=0}^\infty (P^t_{vw} - P^t_{uw}),
$$
and the statement of the lemma follows by combining everything together.
\end{proof}

The connection between the Moore-Penrose pseudoinverse of Laplacians and the Green's functions of random walks in Lemma \ref{lem:Laplacians and random walks} is the key result that allows us to derive spectral bounds in terms of the second largest eigenvalue in magnitude of the transition matrix $P$. We now present three lemmas that, albeit generic, are instrumental to the proof of Theorem \ref{thm:Decay of correlation} in Section \ref{sec:Optimal Network Flow}. Henceforth, let $d$ denote the graph-theoretical distance on $G$: that is, $d(u,v)$ denotes the length of the shortest path between vertex $u$ and vertex $v$. Note that $d(u,v) = \inf\{t\ge 0 : P^t_{uv}\neq 0\}$, as we assumed that to each edge $\{v,w\}\in E$ is associated a non-negative weight $W_{vw}=W_{wv}>0$. Let $n:=|V|$, and let $-1\le\lambda_n \le \lambda_{n-1} \le \cdots \le \lambda_2 < \lambda_1 =1$ be the eigenvalues of $P$. Define $\lambda:=\max\{|\lambda_2|, |\lambda_n|\}$. 

\begin{lemma}
\label{lem:lap and rws}
For any $u,v\in V$ and $f=(f_w)_{w\in V}\in\mathbb{R}^V$ so that $\mathbb{1}^Tf=0$ we have
\begin{align*}
	(e_{u}-e_{v})^T L^{+} f
	&=
	\sum_{w\in V} \sum_{t=0}^\infty (P^t_{uw} - P^t_{vw}) \frac{f_w}{d_w}.
\end{align*}
\end{lemma}

\begin{proof}
From Lemma \ref{lem:Laplacians and random walks}, by summing the quantity $(e_{u}-e_{v})^T L^{+} (e_{w}-e_{z})$ over $z\in V$, recalling that $\sum_{z\in V} e_z = \mathbb{1}$ and $L^+\mathbb{1} = 0$ we have
\begin{align*}
	(e_{u}-e_{v})^T L^{+} e_{w}
	&=
	\sum_{t=0}^\infty (P^t_{uw} - P^t_{vw}) \frac{1}{d_w}
	- \frac{1}{|V|} \sum_{z\in V} \sum_{t=0}^\infty (P^t_{uz} - P^t_{vz}) \frac{1}{d_z}.
\end{align*}
The identity in the statement of the Lemma follows easily as $f=\sum_{w\in V} f_w e_w$ and $\sum_{w\in V} f_w=0$ by assumption. 
\end{proof}

\begin{lemma}
\label{lem:Laplacians and random walks bound}
For any $U, Z\subseteq V$ and any $(f_z)_{z\in Z}\in\mathbb{R}^Z$ we have
\begin{align*}
	\sqrt{
	\frac{1}{2} \sum_{u,v\in U :\{u,v\}\in E}
	\left ( \sum_{z\in Z} \sum_{t=0}^\infty (P^t_{uz} - P^t_{vz}) f_z \right)^2
	}
	\le \alpha
	\frac{\lambda^{d(U,Z)}}{1-\lambda}
	\sqrt{\sum_{z\in Z} f_z^2 d_z},
\end{align*}
with $\alpha:= \frac{\max_{u\in U} \sqrt{2 |\mathcal{N}(u) \cap U|}}{\min_{u\in U} \sqrt{d_u}}$.
\end{lemma}

\begin{proof}
Consider the matrix $\Gamma := D^{1/2} P D^{-1/2} = D^{-1/2} W D^{-1/2}$. This matrix is clearly similar to $P$ and symmetric. Let denote by $\psi_n,\ldots,\psi_1$ the orthonormal eigenvectors of $\Gamma$ corresponding, respectively, to the eigenvalues $\lambda_n \le \lambda_{n-1} \le \cdots \le \lambda_2 \le \lambda_1$. By substitution, it is easy to check that $\sqrt{\pi}\equiv(\sqrt{\pi_v})_{v\in V}$ is an eigenvector of $\Gamma$ with eigenvalue equal to $1$, where we recall that $\pi_v=d_v/\sum_{v\in V}d_v$. Since this eigenvector has positive entries, it follows by the Perron-Frobenius theory that $-1 \le \lambda_n \le \lambda_{n-1} \le \cdots \le \lambda_2 < \lambda_1 = 1$ and that $\psi_1=\sqrt{\pi}$. As $\Gamma$ admits the spectral form $\Gamma=\sum_{k=1}^n \lambda_k \psi_k\psi_k^T$, by the orthonormality of the eigenvectors we have, for $t\ge 0$, $u,z\in V$,
$$
	P^t_{uz} = (D^{-1/2}\Gamma^t D^{1/2})_{uz}
	= \sum_{k=1}^n \lambda_k^t (D^{-1/2} \psi_k\psi_k^T D^{1/2})_{uz}
	= \pi_z + \sum_{k=2}^n \lambda^t_k \psi_{ku} \psi_{kz} \sqrt{\frac{d_z}{d_u}},
$$
where $\psi_{ku}\equiv (\psi_k)_u$ is the $u$-th component of $\psi_k$.
As $P^t_{uz}=0$ whenever $d(u,z)>t$, we have $P^t_{uz}-P^t_{vz} = \mathbf{1}_{d(U,Z) \le t} (P^t_{uz}-P^t_{vz})$ for any $u,v\in U, z\in Z$. Hence, for any $u,v\in U$, let
\begin{align*}
	g_{uv} := \sum_{z\in Z} \sum_{t=0}^\infty (P^t_{uz} - P^t_{vz}) f_z
	&= 
	\sum_{k=2}^n
	\left(
	\frac{\psi_{ku}}{\sqrt{d_u}}
	-
	\frac{\psi_{kv}}{\sqrt{d_v}}
	\right)
	\sum_{z\in Z} 
	\psi_{kz} \sqrt{d_z} f_z
	\sum_{t=d(U,Z)}^\infty \lambda^t_k.
\end{align*}
As $\lambda <1$ by assumption, the geometric series converges for any $k\neq 1$. If we define the quantity
$h_{u}:=
	\sum_{k=2}^n
	\frac{\lambda_k^{d(U,Z)}}{1-\lambda_k}
	\frac{\psi_{ku}}{\sqrt{d_u}}
	\sum_{z\in Z} \psi_{kz} \sqrt{d_z} f_z
$	
for each $u\in V$, we have $g_{uv} = h_u - h_v$, and 
the triangle inequality for the $\ell_2$-norm yields
\begin{align*}
	\sqrt{\sum_{u,v\in U :\{u,v\}\in E} g_{uv}^2}
	&\le 
	2 \sqrt{\sum_{u,v\in U :\{u,v\}\in E} 
	h_{u}^2
	}
	\le
	2 \sqrt{\max_{u\in U} |\mathcal{N}(u) \cap U|} \sqrt{\sum_{u\in U} h_u^2},
\end{align*}
where the factor $2$ comes by the symmetry between $u$ and $v$. Expanding the squares and using that $|\lambda_k| \le \lambda$ for each $k\neq 1$, we get
\begin{align*}
	d_u h_{u}^2
	\le&\
	\frac{\lambda^{2d(U,Z)}}{(1-\lambda)^2}
	\sum_{k=1}^n
	\psi^2_{ku}
	\left(\sum_{z\in Z} \psi^2_{kz} d_z f^2_z
	+ \sum_{z,z'\in Z: z\neq z'} \psi_{kz}\psi_{kz'} \sqrt{d_zd_{z'}} f_zf_{z'} \right)\\
	&\ + \sum_{k,k'\in\{2,\ldots,n\}:k\neq k'}
	\frac{\lambda_k^{d(U,Z)}}{1-\lambda_k}
	\frac{\lambda_{k'}^{d(U,Z)}}{1-\lambda_{k'}}
	\psi_{ku}\psi_{k'u}
	\sum_{z,z'\in Z} \psi_{kz} \psi_{k'z'} \sqrt{d_zd_{z'}} f_z f_{z'},
\end{align*}
where we also used that $\sum_{k=2}^n x_k \le \sum_{k=1}^n x_k$ if $x_1,\ldots,x_n$ are non-negative numbers.
Let $\Psi$ denote the matrix having the eigenvectors $\psi_1,\ldots,\psi_n$ in its columns, namely, $\Psi_{uk}=(\psi_{k})_u=\psi_{ku}$. This is an orthonormal matrix, so both its columns and rows are orthonormal, namely, $\sum_{u=1}^n \psi_{ku} \psi_{k'u} = \mathbf{1}_{k=k'}$ and $\sum_{k=1}^n \psi_{ku} \psi_{kv} = \mathbf{1}_{u=v}$. Using this facts, it is easy to check that
$$
	\sum_{u\in U} h^2_u \le 
	\frac{1}{\min_{u\in U} d_u} \sum_{u\in V} d_u h^2_u
	\le 
	\frac{1}{\min_{u\in U} d_u} \frac{\lambda^{2d(U,Z)}}{(1-\lambda)^2}
	\sum_{z\in Z} d_z f_z^2,
$$
and the proof follows easily by putting all the pieces together, realizing that the quantity that is upper-bounded in the statement of the lemma corresponds to $\frac{1}{\sqrt{2}}(\sum_{u,v\in U :\{u,v\}\in E} g_{uv}^2)^{1/2}$.
\end{proof}

\begin{lemma}
\label{lem:laplacian rw spectral bound ok}
For any $U, Z\subseteq V$ and any $(f_z)_{z\in Z}\in\mathbb{R}^Z$ such that $\sum_{z\in Z} f_z = 0$, we have
\begin{align*}
	\sqrt{
	\frac{1}{2} \sum_{u,v\in U :\{u,v\}\in E}
	( (e_{u}-e_{v})^T L^{+} f )^2
	}
	\le \gamma
	\frac{\lambda^{d(U,Z)}}{1-\lambda}
	\sqrt{\sum_{z\in Z} f_z^2},
\end{align*}
with $\gamma:= \frac{\max_{u\in U} \sqrt{2 |\mathcal{N}(u) \cap U|}}{\min_{u\in U} d_u}$.
\end{lemma}

\begin{proof}
It follows immediately from Lemma \ref{lem:lap and rws} and Lemma \ref{lem:Laplacians and random walks bound}.
\end{proof}

We are now ready to present the proof of Theorem \ref{thm:Decay of correlation} in Section \ref{sec:Optimal Network Flow}.\\

\begin{proof}[Proof of Theorem \ref{thm:Decay of correlation}]
Fix $\varepsilon\in\mathbb{R}$. From Lemma \ref{lem:comparisontheoremnetworkflow} we have
$
	\frac{d x^\star(b(\varepsilon))}{d \varepsilon} 
	= \Sigma(b(\varepsilon))A^T L(b(\varepsilon))^{+} \frac{d b(\varepsilon)}{d \varepsilon}
$
or, entry-wise, for any $(u,v)\in \vec{E}$,
$$
	\frac{d x^\star(b(\varepsilon))_{(u,v)}}{d \varepsilon} 
	= W(b(\varepsilon))_{uv} (e_u-e_v)^T L(b(\varepsilon))^{+} \frac{d b(\varepsilon)}{d \varepsilon}.
$$
Let $(U,F)$ be the undirected graph naturally associated to $(U,\vec{F})$ (see Remark \ref{rem:notation}). Clearly,
\begin{align*}
	\sqrt{\sum_{e\in \vec{F}} \left(\frac{d x^\star(b(\varepsilon))_e}{d \varepsilon}\right)^2}
	\le
	\max_{u,v\in U} W(b(\varepsilon))_{uv}
	\sqrt{
	\frac{1}{2} \sum_{u,v\in V' :\{u,v\}\in F}	
	\left(
	(e_{u}-e_{v})^T L(b(\varepsilon))^{+} \frac{d b(\varepsilon)}{d \varepsilon}
	\right)^2
	},
\end{align*}
and, upon choosing $f=\frac{d b(\varepsilon)}{d \varepsilon}$ in Lemma \ref{lem:laplacian rw spectral bound ok}, we obtain
$$
	\left\| \frac{d x^\star(b(\varepsilon))}{d \varepsilon} \right\|_{\vec{F}}
	\le c(b(\varepsilon))\,
	\frac{\lambda(b(\varepsilon))^{d(U,Z)}}{1-\lambda(b(\varepsilon))}
	\left\|\frac{d b(\varepsilon)}{d \varepsilon}\right\|_Z,
$$
where $c(b):= \frac{\max_{v\in U} \sqrt{2 |\mathcal{N}(v) \cap U|}}{\min_{v\in U} d(b)_v} \max_{u,v\in U} W(b)_{uv}$, for any $b\in\operatorname{Im}(A)$. The proof follows immediately by taking suprema over $b\in\operatorname{Im}(A)$ and $\varepsilon\in\mathbb{R}$.
\end{proof}


\section{Proof of Theorem \ref{thm:error localized} in Section \ref{sec:Scale-free algorithm}}
\label{app:error localized}

This appendix is devoted to the proof of Theorem \ref{thm:error localized} in Section \ref{sec:Scale-free algorithm}, which relies on the decay of correlation property established in Theorem \ref{thm:Decay of correlation} for the min-cost network flow problem. Recall that the constants appearing in the bounds in Theorem \ref{thm:error localized} do not depend on the choice of the subgraph $\vec G'$ of $\vec G$, but depend only on $\mu$, $Q$, $k_+$, and $k_-$. To be able to prove this type of bounds, we first need to develop estimates to relate the eigenvalues of weighted subgraphs to the eigenvalues of the corresponding unweighted graph.

\subsection{Eigenvalues interlacing}
\label{sec:Eigenvalues Interlacing}
Let $G=(V,E)$ be a simple (i.e., no self-loops, and no multiple edges), connected, undirected graph, with vertex set $V$ and edge set $E$. Let $B\in\mathbb{R}^{V\times V}$ be the vertex-to-vertex adjacency matrix of the graph, which is the symmetric matrix defined as
$$
	B_{uv}:=
	\begin{cases}
	1 &\text{if } \{u,v\}\in E,\\
	0 &\text{otherwise}.
	\end{cases}
$$
If $n:=|V|$, denote by $\mu_{n} \le \mu_{n-1} \le \cdots \le \mu_2 \le \mu_1$ the eigenvalues of $B$.
Let $G'=(V', E')$ be a connected subgraph of $G$. Assume that to each edge $\{u,v\}\in E'$ is associated a non-negative weight $W_{uv}=W_{vu}>0$, and let $W_{uv}=0$ if $\{u,v\}\not\in E$. Let $D'$ be a diagonal matrix with entries $D'_{vv}=\sum_{w\in V'} W'_{vw}$ for each $v\in V'$. Let $P':=D'^{-1}W'$. If $m:=|V'|$, denote by $\lambda'_{m} \le \lambda'_{m-1} \le \cdots \le \lambda'_2 \le \lambda'_1$ the eigenvalues of $P'$. The following proposition relates the eigenvalues of $P'$ to the eigenvalues of $B$. In particular, we provide a bound for the second largest eigenvalue in magnitude of $P'$ with respect to the second largest eigenvalue in magnitude of $B$, uniformly over the choice of $G'$.

\begin{proposition}[Eigenvalues interlacing]
\label{prop:interlacing}
For each $\{v,w\}\in E$, let $w_- \le W_{vw}\le w_+$ for some positive constants $w_-$ and $w_+$. Let $k_-$ and $k_+$ be, respectively, the minimum and maximum degree of $G$. Then,
$$
	1 - \frac{w_+k_+}{w_-k_-} + \frac{w_+}{w_-k_-} \mu_{i+n-m}
	\le
	\lambda'_i
	\le 1 - \frac{w_-k_-}{w_+k_+} + \frac{w_-}{w_+k_+} \mu_i.
$$
Therefore, if $\lambda':=\max\{|\lambda'_2|,|\lambda'_{m}|\}$ and $\mu:=\max\{|\mu_2|,|\mu_{n}|\}$, we have
$$
	\lambda' \le \frac{w_+k_+}{w_-k_-} -1 + \frac{w_+}{w_- k_-}\mu.
$$
\end{proposition}

\begin{proof}
Consider the matrix $\Gamma' := D'^{1/2} P' D'^{-1/2} = D'^{-1/2} W' D'^{-1/2}$. As this matrix is similar to $P'$, it shares the same eigenvalues with $P'$. Let $L':=D'-W'$ be the Laplacian associated to $G'$. The Courant-Fischer Theorem yields
$$
	\lambda'_i = \max_{\substack{S\subseteq\mathbb{R}^m\\\operatorname{dim}(S)=i}} \min_{x\in S} \frac{x^T \Gamma' x}{x^Tx}
	= 1 + \max_{\substack{S\subseteq\mathbb{R}^m\\\operatorname{dim}(S)=i}} \min_{y\in S} \frac{-y^T L' y}{y^TD'y},
$$
where we used that $x^T \Gamma' x= x^Tx - y^TL'y$ with $y:=D'^{-1/2}x$, and that the change of variables $y=D'^{-1/2}x$ is non-singular (note that as $G'$ is connected, then $D'$ has non-zero entries on the diagonal). The Laplacian quadratic form yields
$$
	y^T L' y = \frac{1}{2} \sum_{u,v\in V'} W_{uv} (y_u-y_v)^2
	\le w_+ \frac{1}{2} \sum_{u,v\in V'} B_{uv} (y_u-y_v)^2
	= w_+ y^T \mathcal{L}' y,
$$
where $\mathcal{L}'$ is the Laplacian of the \emph{unweighted} graph $G'=(V',E')\equiv(V',E',B')$ with $B':=B_{V',V'}$. Note that we have $\mathcal{L}'=K'-B'$, where $K'$ is diagonal and $K'_{vv}=\sum_{w\in V'}B'_{vw}$ is the degree of vertex $v\in V'$ in $G'$. As $y^TK'y = \sum_{v\in V'} K_{vv} y^2_v \le k_+ y^Ty$, we have
$$
	y^T L' y
	\le w_+ k_+ y^T y - w_+ y^T B' y.
$$
At the same time, $y^TD'y \ge w_-k_- y^Ty$. Therefore,
$$
	\lambda'_i
	\ge 1 - \frac{w_+k_+}{w_-k_-} + \frac{w_+}{w_-k_-} \max_{\substack{S\subseteq\mathbb{R}^m\\\operatorname{dim}(S)=i}} \min_{y\in S} \frac{y^T B' y}{y^Ty}
	= 1 - \frac{w_+k_+}{w_-k_-} + \frac{w_+}{w_-k_-} \mu'_i,
$$
where $\mu'_{n} \le \mu'_{n-1} \le \cdots \le \mu'_2 \le \mu'_1$ are the eigenvalues of $B'$, and the equality follows from the Courant-Fischer Theorem. Analogously, it is easy to prove that
$
	\lambda'_i
	\le 1 - \frac{w_-k_-}{w_+k_+} + \frac{w_-}{w_+k_+} \mu'_i.
$
As $B'$ is a principal submatrix of $B$, the eigenvalue interlacing theorem for symmetric matrices yields $\mu_{i+n-m}\le\mu'_i\le \mu_i$, and we have
$
	\alpha + \beta \mu_{i+n-m}
	\le
	\lambda'_i
	\le \gamma + \delta \mu_i,
$
with $\alpha:=1 - \frac{w_+k_+}{w_-k_-}, \beta:=\frac{w_+}{w_-k_-}, \gamma := 1 - \frac{w_-k_-}{w_+k_+},$ and $\delta:=\frac{w_-}{w_+k_+}$. Clearly,
$
	|\lambda'_i|
	\le \max\{|\alpha + \beta \mu_{i+n-m}|, |\gamma + \delta \mu_i| \}
	\le -\alpha + \beta \max\{ |\mu_{i+n-m}|, |\mu_i| \},
$
so that
$$
	\max\{|\lambda'_2|,|\lambda'_{m}|\}
	\le -\alpha + \beta \max\{
	|\mu_{2+n-m}|, |\mu_2|,|\mu_{n}|, |\mu_m|
	\}
	= -\alpha + \beta \max\{|\mu_2|,|\mu_{n}|\}.
$$
\end{proof}

\subsection{Proof of Theorem \ref{thm:error localized}}

We are now ready to present the proof of Theorem \ref{thm:error localized}. The proof relies on repeatedly applying Theorem \ref{thm:Decay of correlation} in Section \ref{sec:Optimal Network Flow} (which captures the decay of correlation for the min-cost network flow problem) and the fundamental theorem of calculus.\\

\begin{proof}[Proof of Theorem \ref{thm:error localized}]
Consider the setting of Section \ref{sec:Scale-free algorithm}.\\
\textbf{Analysis of the bias term.}\\
Let us first bound the bias outside $\vec E'$. Let $n:=|V|$, and for each $b\in\operatorname{Im}(A)$ let $-1\le\lambda_n(b) \le \lambda_{n-1}(b) \le \cdots \le \lambda_2(b) < \lambda_1(b) =1$ be the eigenvalues of $P(b)$. Let $\lambda(b):=\max\{|\lambda_2(b)|, |\lambda_n(b)|\}$ and $\lambda:=\sup_{b\in\operatorname{Im}(A)} \lambda(b)$. Define $b(\varepsilon):=b+\varepsilon p$, for any non-negative real number $\varepsilon \ge 0$. If $e\in \vec E'^C$, then $T'_{b(\varepsilon)}(x^\star(b))_e = x^\star(b)_e$ and
\begin{align*}
	\operatorname{Bias}(p,\vec G')_e
	&= x^\star(b(1))_e - x^\star(b(0))_e
	= \int_0^1 d\varepsilon \, 
	\frac{d x^\star(b(\varepsilon))_e}{d \varepsilon}.
\end{align*}
By the triangle inequality for the $\ell_2$-norm and Theorem \ref{thm:Decay of correlation}, we obtain
$$
	\left\| \operatorname{Bias}(p,\vec G') \right\|_{\vec E'^C}
	\le 
	\int_0^1 d\varepsilon \, 
	\left\| \frac{d x^\star(b(\varepsilon))}{d \varepsilon} \right\|_{\vec E'^C}
	\le \sup_{\varepsilon\in\mathbb{R}} \left\| \frac{d x^\star(b(\varepsilon))}{d \varepsilon} \right\|_{\vec E'^C}
	\le c\,
	\| p \|\,
	\frac{\lambda^{d(\Delta(\vec G'),Z)}}{1-\lambda},
$$
where for the last inequality we used that $\sup_{\varepsilon\in\mathbb{R}} \|\frac{d b(\varepsilon)}{d \varepsilon}\|_Z= \|p\|$, as $\frac{d b(\varepsilon)_v}{d\varepsilon}=p_v$ for $v\in Z$ and $\frac{d b(\varepsilon)_v}{d\varepsilon}=0$ for $v\not\in Z$, and where $c:=\sqrt{2k_+} Q/k_-$.

Let us now consider the bias inside $\vec E'$.
Let $A':=A_{V',\vec E'}\in\mathbb{R}^{V'\times \vec E'}$ be the vertex-edge adjacency matrix of the subgraph $\vec G'$. For $b'\in\operatorname{Im}(A')\subseteq\mathbb{R}^{V'}$, consider the following optimization problem over $x'\in\mathbb{R}^{\vec E'}$:
\begin{align*}
\begin{aligned}
	\text{minimize }\quad   & f'(x'):=\sum_{e\in \vec E'} f_e(x'_e)\\
	\text{subject to }\quad & A' x' = b',
\end{aligned}
\end{align*}
and denote its unique optimal point as the function
$$
	x'^\star:b'\in \operatorname{Im}(A')\subseteq\mathbb{R}^{V'}\longrightarrow
	x'^\star(b') := {\arg\min}\left\{ f'(x') : x'\in\mathbb{R}^{\vec E'},A' x' = b' \right\}
	\in\mathbb{R}^{\vec E'}.
$$
For any $\varepsilon>0,\theta>0$, define $b'(\varepsilon,\theta)\in\mathbb{R}^{V'}$ as
\begin{align}
	b'(\varepsilon,\theta)
	:= b(\varepsilon)_{V'} - A_{V',\vec E'^C} x^\star(b(\theta))_{\vec E'^C}.
	\label{def:b large}
\end{align}
Without loss of generality, we can index the elements of $V'$ and $\vec E'$ so that the matrix $A$ has the following block structure:
$$
	A
	=
	\left( \begin{array}{cc}
	A_{V',\vec E'} & A_{V',\vec E'^C} \\
	A_{V'^C,\vec E'} & A_{V'^C,\vec E'^C}
	\end{array} \right)
	=
	\left( \begin{array}{cc}
	A' & A_{V',\vec E'^C} \\
	\mathbb{0} & A_{V'^C,\vec E'^C}
	\end{array} \right).
$$
For any $x$ that satisfies the flow constraints on $\vec E'^C$ with respect to $b(\varepsilon)$, namely, $A_{V\setminus V', \vec E'^C}x_{\vec E'^C}=b(\varepsilon)_{V\setminus V'}$, we clearly have
\begin{align}
	\left( \lim_{t\rightarrow\infty} T'^t_{b+p}(x) \right)_{\vec E'}
	&= \arg\min
	\left\{
	f'(x')
	:
	x'\in\mathbb{R}^{\vec E'},A (x'x_{\vec E'^C}) = b(1)
	\right\}\nonumber\\
	&= 
	\arg\min
	\left\{
	f'(x')
	:
	x'\in\mathbb{R}^{\vec E'},A' x' = 
	b(1)_{V'} - A_{V',\vec E'^C} x_{\vec E'^C}
	\right\}\nonumber\\
	&\equiv x'^\star( b(1)_{V'} - A_{V',\vec E'^C} x_{\vec E'^C} ).\nonumber
\end{align}
Clearly $x^\star(b)$ satisfies the flow constraints on $\vec E'^C$ with respect to $b(1)$, as $p$ is supported on $V'$ so that $b(\varepsilon)_{V'^C} = b_{V'^C}$. Recalling the definition of $b'(\varepsilon,\theta)$ in \eqref{def:b large}, we then have
\begin{alignat*}{2}
	\left(\lim_{t\rightarrow\infty} T'^t_{b+p}(x^\star(b))\right)_{\vec E'}
	= x'^\star(b'(1,0)).
\end{alignat*}
On the other hand, as $x^\star(b(1))$ is clearly a fixed point of the map $T'_{b(1)}$, we can characterize the components of $x^\star(b(1))$ supported on $\vec E'$ as
\begin{align*}
	x^\star(b(1))_{\vec E'}
	&=
	\left(\lim_{t\rightarrow\infty} T'^t_{b(1)}(x^\star(b(1))\right)_{\vec E'}
	= x'^\star(b'(1,1)).
\end{align*}
It is easy to check that $b'(\varepsilon,\theta)\in\operatorname{Im}(A')$ for each value of $\varepsilon$ and $\theta$. In fact, as $\vec G'$ is connected by assumption, then $\operatorname{Im}(A')$ corresponds to the subspace of $\mathbb{R}^{V'}$ orthogonal to the all-ones vector $\mathbb{1}$. We have
$
	\mathbb{1}^T b'(\varepsilon,\theta)
	= \mathbb{1}^Tb_{V'} + \varepsilon \mathbb{1}^Tp_{V'} - \mathbb{1}^TA_{V',\vec E'^C} x^\star(b(\theta))_{\vec E'^C}.
$
Note that $\mathbb{1}^Tp_{V'} = 0$ by assumption. Also, $0=\mathbb{1}^Tb = \mathbb{1}^Tb_{V'} + \mathbb{1}^Tb_{V'^C}$ (note the different dimension of the all-ones vectors) so that $\mathbb{1}^Tb_{V'} = -\mathbb{1}^Tb_{V'^C}$. Analogously, as $\mathbb{1}^TA=\mathbb{0}^T$, we have $\mathbb{1}^TA_{V',\vec E'^C} = - \mathbb{1}^TA_{V'^C,\vec E'^C}$
Hence, 
$$
	\mathbb{1}^T b'(\varepsilon,\theta)
	= -\mathbb{1}^Tb_{V'^C} + \mathbb{1}^TA_{V'^C,\vec E'^C} x^\star(b(\theta))_{\vec E'^C}
	= \mathbb{0}^T,
$$
where the last equality follows as clearly $A_{V'^C,\vec E'^C} x^\star(b(\theta))_{\vec E'^C} = b_{V'^C}$. Therefore, we have
\begin{align*}
	\operatorname{Bias}(p,\vec G')_e
	&= 
	x'^\star(b'(1,1))_e - 
	x'^\star(b'(1,0))_e
	= \int_0^1 d\theta  
	\, 
	\frac{d x'^\star(b'(1,\theta))_e}{d \theta}.
\end{align*}
For each $b'\in\operatorname{Im}(A')$, let $W'(b')\in\mathbb{R}^{V'\times V'}$ be a symmetric matrix defined as
\begin{align*}
	W'(b')_{uv}
	:=
	\begin{cases}
	\left(\frac{\partial^2 f_e(x'^\star(b')_e)}
	{\partial x_e^2}\right)^{-1} & \text{if } e=(u,w) \text{ or } e=(w,u) \in \vec E,\\
	0 & \text{otherwise},
	\end{cases}
\end{align*}
and let $D'(b')\in\mathbb{R}^{V'\times V'}$ be a diagonal matrix with entries $D'(b')_{vv}=\sum_{w\in V'} W'(b')_{vw}$. Let $P'(b'):=D'(b')^{-1}W'(b')$. If $m:=|V|$, let $-1\le\lambda'_m(b') \le \lambda'_{m-1}(b') \le \cdots \le \lambda'_2(b') < \lambda'_1(b') =1$ be the eigenvalues of $P'(b')$ (where this characterization holds as $G'$ is connected by assumption). Define $\lambda'(b'):=\max\{|\lambda'_2(b')|, |\lambda'_m(b')|\}$ and $\lambda':=\sup_{b'\in\operatorname{Im}(A')} \lambda'(b')$. 
Proceeding as above, applying Theorem \ref{thm:Decay of correlation} to the optimization problem defined on $\vec G'$ (recall that $G'$ is connected by assumption), we get
$$
	\left\| \operatorname{Bias}(p,\vec G') \right\|_{\vec E'}
	\le \sup_{\theta\in\mathbb{R}} \left\| \frac{d x'^\star(b'(1,\theta))}{d \theta} \right\|_{\vec E'}
	\le
	c \, \frac{1}{1-\lambda'}
	\sup_{\theta\in\mathbb{R}} 
	\left\| \frac{\partial b'(1,\theta)}{\partial\theta} \right\|_{\Delta(\vec G')},
$$
where we used that $\frac{\partial b'(\varepsilon,\theta)_v}{\partial\theta}=0$ if $v\in V'\setminus \Delta(\vec G')$, and clearly $d(V',\Delta(\vec G'))=0$ as $\Delta(\vec G')\subseteq V'$.
For $v\in \Delta(\vec G')$ we have
$
	\frac{\partial b'(\varepsilon,\theta)_v}{\partial\theta}
	=
	- \sum_{e\in \vec E'^C} A_{ve} 
	\frac{d x^\star(b(\theta))_{e}}{d \theta}.
$
If $\vec F(v):=\{e\in \vec E: e=(u,v) \text{ or } e=(v,u), e\in \vec E'^C\}$, Jensen's inequality yields
\begin{align*}
	\left(\frac{\partial b'(1,\theta)_v}{\partial\theta}\right)^2
	&\le
	\left(\sum_{e\in \vec F(v)}  
	\left|\frac{d x^\star(b(\theta))_{e}}{d \theta}\right|
	\right)^2
	=
	|\vec F(v)|^2\left(\sum_{e\in \vec F(v)} \frac{1}{|\vec F(v)|}
	\left|\frac{d x^\star(b(\theta))_{e}}{d \theta}\right|
	\right)^2
	\\
	&\le
	|\vec F(v)|
	\sum_{e\in \vec F(v)}
	\left(\frac{d x^\star(b(\theta))_{e}}{d \theta}\right)^2.
\end{align*}
As $\max_{v\in \Delta(\vec G')}|\vec F(v)|\le k_+-1$, applying Theorem \ref{thm:Decay of correlation} as done above we get
$$
	\left\| \frac{\partial b'(1,\theta)}{\partial\theta} \right\|_{\Delta(\vec G')}
	\le \sqrt{k_+-1}
	\left\| \frac{d x^\star(b(\theta))}{d\theta} \right\|_{\vec E'^C}
	\le
	c\sqrt{k_+-1}\,\| p \|\,
	\frac{\lambda^{d(\Delta(\vec G'),Z)}}{1-\lambda}.
$$
Therefore,
$
	\left\| \operatorname{Bias}(p,\vec G') \right\|_{\vec E'}
	\le
	c^2\sqrt{k_+-1}\,\| p \|\,
	\frac{\lambda^{d(\Delta(\vec G'),Z)}}{(1-\lambda')(1-\lambda)}.
$
By the triangle inequality for the $\ell_2$-norm we have $\|\operatorname{Bias}(p,\vec G')\|\le \|\operatorname{Bias}(p,\vec G')\|_{\vec E'} + \|\operatorname{Bias}(p,\vec G')\|_{\vec E'^C}$, so we obtain
$$
	\left\| \operatorname{Bias}(p,\vec G') \right\|
	\le 
	c
	\left( 1
	+
	c\sqrt{k_+-1}
	\right)
	\left\| p \right\|
	\frac{\lambda^{d(\Delta(\vec G'),Z)}}{(1-\lambda')(1-\lambda)}.
$$
By Proposition \ref{prop:interlacing} we have
$
	\max\{\lambda,\lambda'\} \le \frac{Qk_+}{k_-} -1 + \frac{Q}{k_-} \mu,
$
and the bound for the bias term follows.

\textbf{Analysis of the variance term.}\\
As
$
	(\lim_{t\rightarrow\infty} T'^t_{b+p}(x^\star(b)))_{\vec E'}
	= x'^\star(b'(1,0)),
$
we have
$$
	\| \operatorname{Variance}(p,\vec G',t) \|_{\vec E'}
	= \
	\| x'^\star(b'(1,0)) - T'^t_{b+p}(x^\star(b))_{\vec E'} \|
	\le e^{- \frac{t}{2Q}} \| x'^\star(b'(1,0)) - x'^\star(b'(0,0)) \|,
$$
where in the last inequality we used that $x^\star(b)_{\vec E'}=x'^\star(b'(0,0))$. For each $e\in \vec E'$ we have
\begin{align*}
	x'^\star(b'(1,0))_e - 
	x'^\star(b'(0,0))_e
	= \int_0^1 d\varepsilon
	\, 
	\frac{d x'^\star(b'(\varepsilon,0))_e}{d \varepsilon},
\end{align*}
and using the triangle inequality for the $\ell_2$-norm, applying Theorem \ref{thm:Decay of correlation} to the optimization problem defined on $\vec G'$, we obtain
$$
	\left\| \operatorname{Variance}(p,\vec G',t) \right\|_{\vec E'}
	\le 
	\int_0^1 d\varepsilon \, 
	\left\| \frac{d x'^\star(b'(\varepsilon,0))}{d \varepsilon} \right\|_{\vec E'}
	\le \sup_{\varepsilon\in\mathbb{R}} \left\| \frac{d x'^\star(b'(\varepsilon,0))}{d \varepsilon} \right\|_{\vec E'}
	\le c\,
	\| p \|\,
	\frac{1}{1-\lambda'},
$$
where we used that $\frac{\partial b'(\varepsilon,0)_v}{\partial \varepsilon}=\frac{d b(\varepsilon)_v}{d \varepsilon}=p_v$ for $v\in Z$ and $\frac{d b(\varepsilon)_v}{d\varepsilon}=0$ for $v\not\in Z$, and that $d(V',Z)=0$ as $Z\subseteq V'$. Clearly, $\operatorname{Variance}(p,\vec G',t)_e=0$ for $e\in \vec E'^C$, as $T'_b(x^\star(b))_e = x^\star(b)_e$. Hence, $\left\| \operatorname{Variance}(p,\vec G',t) \right\|=\left\| \operatorname{Variance}(p,\vec G',t) \right\|_{\vec E'}$ and the proof is concluded as $\lambda' \le \frac{Qk_+}{k_-} -1 + \frac{Q}{k_-} \mu$ by Proposition \ref{prop:interlacing}.
\end{proof}

\end{document}